\documentclass[11pt,letterpaper]{article}

\usepackage[margin=1in]{geometry}
\usepackage{amsmath,amssymb,amsthm}
\usepackage{graphicx}
\usepackage{booktabs}
\usepackage{natbib}
\usepackage{hyperref}
\usepackage{cleveref}
\usepackage{algorithm}
\usepackage{algpseudocode}
\usepackage{xcolor}
\usepackage{enumitem}
\usepackage{subcaption}
\usepackage{tikz}
\usetikzlibrary{arrows.meta,positioning,shapes.geometric,calc,fit,backgrounds,decorations.pathreplacing}

\newtheorem{definition}{Definition}[section]
\newtheorem{theorem}{Theorem}[section]

\newtheorem{proposition}[theorem]{Proposition}

\newtheorem{example}{Example}[section]

\newcommand{\SGH}{\textsc{Graph Harness}}
\newcommand{\ExecSystem}{\mathcal{E}}
\renewcommand{\State}{\mathcal{S}}
\newcommand{\Units}{\mathcal{U}}
\newcommand{\Policy}{\mathcal{P}}
\newcommand{\Obs}{\mathcal{O}}
\newcommand{\Update}{\Delta}
\newcommand{\ready}{\mathit{ready}}
\newcommand{\ExecPlan}{\Pi}

\title{From Agent Loops to Structured Graphs:\\A Scheduler-Theoretic Framework for LLM Agent Execution\\[1em]}

\author{
  \textbf{Hu Wei}\\[0.5em]
  \texttt{1990huwei@sina.com}\\[1em]
}

\date{}

\begin{document}
\maketitle

\begin{abstract}
The dominant paradigm for building LLM-based agents is the \emph{Agent Loop}---an iterative cycle where a single language model decides what to do next by reading an ever-growing context window. This paradigm has three structural weaknesses: implicit dependencies between steps, unbounded recovery loops that may retry indefinitely, and mutable execution history that makes debugging difficult. We characterize the Agent Loop as a \emph{single-ready-unit scheduler}: at any instant, at most one executable unit is active, and the choice of which unit to activate is the output of an opaque LLM inference rather than an inspectable policy. This characterization lets us place Agent Loops and graph-based execution engines on a single semantic continuum.

We propose \SGH{} (Structured Graph Harness), which lifts the control structure from implicit context into an explicit static DAG. \SGH{} makes three design commitments: an execution plan is immutable for the duration of a plan version; planning, execution, and recovery are separated into three independent layers; and recovery follows a strict escalation protocol. These commitments trade some expressiveness for controllability, verifiability, and implementability.

Our contributions are fourfold: a scheduler-unified framework that applies classical scheduling theory to LLM agent execution, identifying the specific challenges introduced by non-deterministic LLM nodes; a trade-off analysis of controllability, expressiveness, and implementability across 70 surveyed systems; a formal specification including a node state machine with proven termination and soundness guarantees; and an attributable experimental framework with a seven-group design for future empirical validation.

This is a position paper and design proposal. We contribute a theoretical framework, a design analysis, and an experimental protocol---not a production implementation or empirical results. The design has been verified for internal consistency and state-machine completeness; engineering details and experimental validation are left to future work.

\end{abstract}

\begin{table}[t]
\centering
\caption{Summary of notation used throughout this paper.}
\label{tab:notation}
\begin{tabular}{@{}ll@{}}
\toprule
\textbf{Symbol} & \textbf{Definition} \\
\midrule
$\ExecSystem = (\State, \Units, \Policy, \Obs, \Update)$ & Execution system tuple (Def.~\ref{def:execution-system}) \\
$\Units(s) \subseteq \mathcal{A}$ & Ready set at state $s$ \\
$|\Units(s)|$ & Ready-set cardinality \\
$\ExecPlan = (\mathit{id},\mathit{version},V,E,\sigma,\kappa)$ & Execution plan (Def.~\ref{def:execution-plan}) \\
$\sigma : V \to \text{NodeConfig}$ & Node configuration mapping \\
$\kappa$ & Plan-level output contract \\
$\kappa_v$ & Per-node output contract for node $v$ \\
$\Sigma, \Sigma_{\mathit{term}}$ & Node state set / terminal states (Def.~\ref{def:node-state}) \\
$\mathcal{R}$ & Recovery action set (Def.~\ref{def:recovery-action}) \\
$\mathcal{C}_{\mathit{exec}}, \mathcal{C}_{\mathit{diag}}$ & Execution / diagnostic context (Def.~\ref{def:context-partition}) \\
\bottomrule
\end{tabular}
\end{table}

\section{Introduction}
\label{sec:introduction}

Large language models (LLMs) have enabled a new class of software systems---\emph{LLM agents}---that autonomously decompose tasks, invoke tools, and iterate on solutions. The dominant paradigm for building these agents is the \emph{Agent Loop}: an iterative cycle of reasoning, acting, and observing where a single LLM decides what to do next by reading an ever-growing context window. Despite its simplicity and widespread adoption, this paradigm has three structural weaknesses that become increasingly apparent as tasks grow in complexity.

First, \textbf{dependencies between steps are implicit and unverifiable.} When an agent must ``modify code, then run tests,'' the fact that the second step depends on the first exists only in the context window. The LLM must \emph{remember} this dependency at inference time; there is no structural guard that prevents out-of-order execution. Second, \textbf{failure recovery has no bounded semantics.} When a step fails, the LLM autonomously decides whether to retry, skip, or replan---with no explicit contract specifying which recovery actions are available for which failure types, and no bound on how many attempts may be made. Third, \textbf{the execution plan can be silently rewritten.} If the LLM revises its plan mid-execution, the original plan is overwritten in the context. After execution, it is impossible to reconstruct a faithful audit trail of which plan governed which actions.

These weaknesses are not edge cases. Our analysis of 70 open-source LLM agent projects reveals that 60\% (42 out of 70) adopt the Agent Loop pattern. The survey methodology and detailed results are provided in Appendix~\ref{sec:appendix-survey}. Recent enhancements---planner-augmented loops, graph-structured orchestration, multi-agent decomposition---improve specific aspects but do not fundamentally address the structural problem: the control flow remains implicit, and the execution lacks a stable commitment.

We observe that the Agent Loop is, at its core, a \emph{single-ready-unit scheduler}: at any point during execution, at most one executable unit (tool invocation, sub-task, or reasoning step) is active, and the choice of the next unit is the output of an opaque LLM inference. This reframing lets us place Agent Loops and graph-based execution engines on a single semantic continuum, making their structural differences precise and comparable. The key parameter is the \emph{ready-set cardinality}---how many units are simultaneously eligible for dispatch---and the \emph{policy explicitness}---how deterministic and inspectable the scheduling decision is.

By contrast, graph-based executors can dispatch multiple nodes simultaneously, enabling **parallel execution** (running independent operations concurrently) and **alternative paths** (trying multiple approaches and selecting one that succeeds). We distinguish two forms of parallelism: **constructive parallelism**, where all branches must complete (e.g., reading two files in parallel), and **competitive parallelism**, where only one branch is needed and the rest can be cancelled (e.g., trying two different fixes and stopping after one succeeds). SGH supports constructive parallelism but excludes competitive parallelism for controllability reasons (see Section 6.3).

\paragraph{Relationship to classical scheduling theory.} The formalization of execution systems as tuples $(\State, \Units, \Policy, \Obs, \Update)$ builds on classical DAG scheduling literature~\citep{topcuoglu2002list, cormen2009introduction}. What is novel is not the tuple representation itself, but its application to LLM agent execution and the identification of **LLM-specific challenges** (non-deterministic node output, semantic validation, reasoning errors) that classical schedulers do not address. We discuss these differences in detail in Section 6.5.

Building on this observation, we propose \emph{Graph Harness} (\SGH{}), a structured execution design that lifts the control structure from implicit context into an explicit static directed acyclic graph (DAG). \SGH{} makes three design commitments: (1)~an execution plan is an immutable commitment for the duration of a plan version; (2)~planning, execution, and recovery are separated into three independent layers with well-defined interfaces; and (3)~recovery actions follow a strict escalation protocol that prevents unbounded replanning. These commitments deliberately trade some expressiveness---competitive parallelism, recursive sub-graph expansion, and parent-chain rollback are excluded---for controllability, verifiability, and implementability.

We make the following theoretical and design contributions:
\begin{itemize}
\item A \textbf{scheduler-theoretic framework} that characterizes LLM agent execution systems as schedulers parameterized by ready-set cardinality ($|\Units|$) and policy explicitness. This framework enables precise comparison of Agent Loops, graph-based executors, and intermediate variants, generating testable hypotheses about performance gains along the gradient ($G_{\mathit{plan}} > 0$, $G_{\mathit{scaffold}} > 0$, etc.).
\item A \textbf{four-principle design methodology} derived from analysis of 70 surveyed systems, explicitly characterizing the trade-offs between controllability, expressiveness, and implementability. Each principle is justified by qualitative observations from the survey (e.g., failure-loop behavior was commonly observed among the graph/flow orchestration systems in our dataset, while such behavior was rare in state-machine-based systems).
\item A \textbf{formal execution model} with three-layer separation (planning, execution, recovery), an immutable plan versioning scheme, and a node state machine with theoretically proven termination and soundness guarantees under explicit fairness assumptions.
\item A \textbf{seven-group experimental protocol} that isolates the contribution of each design decision ($G_{\mathit{plan}}$, $G_{\mathit{scaffold}}$, $G_{\mathit{graph}}$, $G_{\mathit{patch}}$, $G_{\mathit{replan}}$). This protocol provides a rigorous methodology for future empirical work; the execution of the experiments themselves is outside the scope of this paper.
\end{itemize}

\medskip
\noindent\textit{Scope.} This paper contributes a theoretical framework, a design analysis, and an experimental protocol---not a production implementation or empirical results. The design has been verified for internal consistency and state-machine completeness; engineering details (concurrent scheduling, distributed logging, fault-tolerant persistence) and experimental validation are the subject of ongoing work.

\paragraph{A motivating example.} Consider a task that requires fixing a bug in a Python project: search two modules in parallel, read both, analyze the root cause, try either of two alternative patches, run tests, update documentation, and generate a report (\Cref{fig:running-example-dag}). In an Agent Loop, all steps execute serially (11~turns), recovery from a failed patch is ad-hoc, and there is no structural guard ensuring that the analysis waits for both files to be read. In \SGH{}, the same task takes 6~scheduling rounds: the two searches and the two reads dispatch in parallel ($|\Units|=2$), the two patches and the documentation update proceed concurrently ($|\Units|=3$), and the any\_of join selects whichever patch succeeds first while skipping the alternative. The parallelism, dependency tracking, and bounded recovery are not LLM decisions---they are structural properties of the DAG.

\begin{figure}[t]
\centering
\begin{tikzpicture}[
  node distance=0.9cm and 1.3cm,
  box/.style={draw, rounded corners, minimum width=1.5cm, minimum height=0.6cm, align=center, font=\scriptsize},
  arr/.style={-{Stealth[length=2mm]}, semithick}
]
  \node[box, fill=blue!10] (sa) {\texttt{search\_auth}};
  \node[box, fill=blue!10, right=of sa] (su) {\texttt{search\_utils}};
  \node[box, fill=blue!10, below=of sa] (ra) {\texttt{read\_auth}};
  \node[box, fill=blue!10, below=of su] (ru) {\texttt{read\_utils}};
  \node[box, fill=orange!15, below right=0.9cm and -0.3cm of ra] (az) {\texttt{analyze}};
  \node[box, fill=green!10, below left=0.9cm and 0.3cm of az] (fa) {\texttt{fix\_A}};
  \node[box, fill=green!10, below=0.9cm of az] (fb) {\texttt{fix\_B}};
  \node[box, fill=yellow!15, below right=0.9cm and 0.3cm of az] (ud) {\texttt{update\_docs}};
  \node[box, fill=green!10, below=0.9cm of fb] (rt) {\texttt{run\_tests}};
  \node[box, fill=orange!15, below right=0.9cm and -0.3cm of rt] (rp) {\texttt{report}};
  \draw[arr] (sa) -- (ra);
  \draw[arr] (su) -- (ru);
  \draw[arr] (ra) -- (az);
  \draw[arr] (ru) -- (az);
  \draw[arr] (az) -- (fa);
  \draw[arr] (az) -- (fb);
  \draw[arr] (az) -- (ud);
  \draw[arr] (fa) -- node[left, font=\tiny] {any\_of} (rt);
  \draw[arr] (fb) -- node[right, font=\tiny] {any\_of} (rt);
  \draw[arr] (rt) -- (rp);
  \draw[arr] (ud) -- (rp);
  \draw[blue!50, dashed, rounded corners=4pt]
    ([shift={(-0.3cm, 0.35cm)}]sa.north west) rectangle ([shift={(0.3cm, -0.5cm)}]su.south east);
  \node[above=0.2cm of sa.north, anchor=south, font=\tiny, blue!50!black] {eligible for parallel ($|\Units|=2$)};
  \draw[green!50, dashed, rounded corners=4pt]
    ([shift={(-0.3cm, 0.35cm)}]fa.north west) rectangle ([shift={(0.3cm, -0.7cm)}]ud.north east);
  \node[above=0.2cm of fa.north, anchor=south, font=\tiny, green!50!black] {eligible for parallel ($|\Units|=3$)};
\end{tikzpicture}
\caption{Motivating example: a bug-fix task as a DAG. Blue nodes form the first parallel wave; green nodes form the second. The any\_of join selects one patch; the other is skipped. \textit{Note:} This example illustrates SGH's potential benefits under ideal conditions: (1)~the planner correctly identifies all parallel structure, (2)~dependencies are fully known at planning time, and (3)~the LLM executes each node correctly. Real-world tasks may not exhibit such clean parallelism, and the planner's ability to recognize independent operations is a prerequisite for these benefits. See \Cref{sec:planning-failures} for discussion of what happens when the planner fails.}
\label{fig:running-example-dag}
\end{figure}
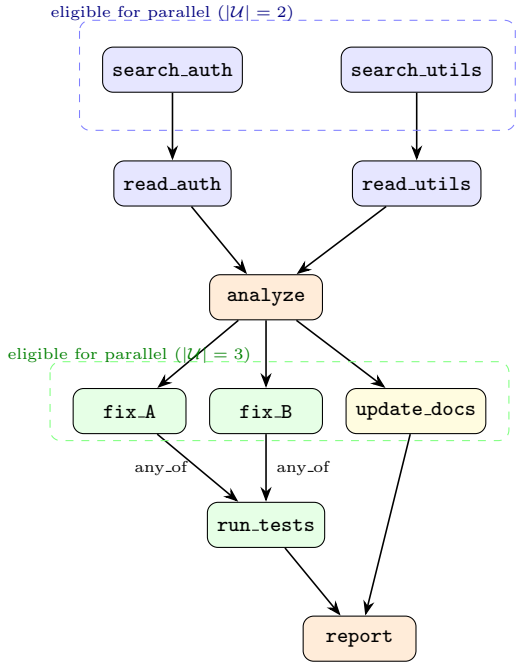

\section{Related Work}
\label{sec:related-work}

\subsection{Agent Loop Paradigm}

The iterative observe--reason--act loop originates in the ReAct framework~\citep{yao2023react} and underlies the majority of production agent systems. Chain-of-Thought prompting~\citep{wei2022cot} showed that structured reasoning traces improve task performance, while Plan-and-Solve~\citep{wang2023plan} extended this with explicit plan-then-execute decomposition. Our survey of 70 open-source projects confirms that approximately 60\% use some variant of the iterative tool loop.

From the scheduler-theoretic perspective we develop in \Cref{sec:scheduler-framework}, all Agent Loop systems share the same structural limitation: $|\Units| = 1$ at every execution step. The scheduling policy is the LLM itself---a black box that is neither deterministic nor interpretable. Our framework makes this limitation precise and comparable to alternatives by parameterizing execution systems along the ready-set cardinality and policy explicitness axes.

\paragraph{\SGH{}'s novel contributions.} Our work differs from the above approaches in three key respects. First, we provide a \emph{scheduler-theoretic framework} that makes the trade-off between expressiveness and controllability \emph{formal}---not just an engineering observation but a consequence of the ready-set cardinality $|\Units|$ and policy explicitness. Second, we deliberately restrict expressiveness (no \texttt{first\_of}, no recursive expansion, no dynamic topology) to maximize controllability and verifiability, a design point that existing graph orchestration systems do not target. Third, we systematically address \emph{LLM-specific challenges} (non-deterministic output, reasoning failures as primary error mode, non-idempotent retry) that classical DAG schedulers do not face. These contributions position \SGH{} as a design point that combines the formal rigor of classical scheduling theory with the practical realities of LLM-powered execution.

\subsection{Plan-Then-Execute and Separated Architectures}

A growing line of work separates planning from execution. Plan-and-Act~\citep{erdogan2025planandact} trains a dedicated Planner model to generate structured high-level plans and an Executor to translate them into actions, demonstrating state-of-the-art performance on web navigation benchmarks. Routine~\citep{mo2025routine} introduces structured planning scripts as an intermediate representation between LLM-generated plans and execution engines, improving multi-step tool-calling accuracy from 41\% to 96\% in enterprise settings. Task-Decoupled Planning (TDP)~\citep{wang2025tdp} decomposes tasks into a DAG of sub-goals with scoped contexts, confining replanning to the active sub-task and reducing token consumption by up to 82\%.

These systems validate the principle that separating planning from execution improves reliability. However, in our framework, they remain single-ready-unit schedulers ($|\Units| = 1$): the plan guides execution but does not create multiple simultaneously dispatchable units. \SGH{} goes further by making the plan itself a scheduling structure with a multi-element ready set.

\paragraph{TDP in detail.} Task-Decomposed Planning~\citep{wang2025tdp} is the closest existing system to \SGH{} and merits detailed comparison. TDP decomposes tasks into a DAG of sub-goals, each with a scoped context that isolates the sub-goal's execution history from the global context. The authors report up to 82\% reduction in token consumption, validating the principle that context isolation improves efficiency.

In our scheduler framework, TDP is classified as a single-ready-unit scheduler ($|\Units| = 1$) based on its published execution model: each sub-task is processed by a dedicated Planner--Executor pair, one at a time, with sequential sub-task scheduling. The DAG structure guides the \emph{order} of execution but does not create multiple simultaneously dispatchable units. TDP's DAG is also \emph{mutable} during execution---the planner can add, remove, or reorder sub-tasks---which conflicts with \SGH{}'s plan-version immutability.

The three key differences are: (1)~\textbf{Ready-set cardinality}: TDP executes one sub-task at a time; \SGH{} dispatches all nodes with satisfied dependencies concurrently. (2)~\textbf{Plan stability}: TDP allows dynamic sub-task graph modification during execution; \SGH{} enforces plan-version immutability---any structural change requires a new plan version via the replan protocol. (3)~\textbf{Recovery}: TDP replans at the sub-task level without a formal escalation protocol; \SGH{} requires exhausting lower-level recovery (retry, then local patch) before permitting full replan, preventing the ``failure loop'' pathology where an LLM repeatedly replans without making progress.

\subsection{LLM-Based Workflow Optimization}

Beyond graph orchestration, a growing body of work focuses on LLM-driven workflow optimization and multi-agent collaboration. \textbf{AutoGen}~\citep{wu2023autogen} introduces conversational agents with role-based message routing, enabling complex multi-agent interactions through structured conversation patterns. While it supports sophisticated agent coordination, it lacks explicit DAG scheduling and structured recovery protocols. \textbf{CrewAI}~\citep{joao2024crewai} implements hierarchical task assignment with role-based agent definitions, providing a clean abstraction for multi-agent workflows. Similar to AutoGen, its recovery mechanism is ad-hoc rather than protocol-driven. \textbf{Semantic Kernel}~\citep{microsoft2023semantic} introduces a planner-executor pattern with composable skills, offering a principled approach to skill reuse. Closest to \SGH{} in principle, Semantic Kernel focuses on skill composition rather than the broader scheduling theory developed here.

These systems represent alternative approaches to the same problem space and demonstrate the value of structured agent coordination. However, in our scheduler framework, they remain primarily single-ready-unit schedulers ($|\Units| = 1$) with implicit or semi-deterministic policies. \SGH{} distinguishes itself by: (1) making the scheduling policy fully explicit and deterministic; (2) enforcing plan-version immutability; and (3) formalizing recovery through a three-level escalation protocol.

\subsection{Graph-Based Agent Orchestration}

Several recent systems model agent workflows as graphs. GPTSwarm~\citep{zhuge2024gptswarm} represents agents as computational graphs with optimizable edges, enabling automatic improvement of agent orchestration. AgentKit~\citep{wu2024agentkit} provides a graph-based prompt composition framework for building complex agent workflows. AFlow~\citep{li2025aflow} models agentic workflows as graphs where nodes represent LLM invocations and edges capture logical dependencies, using Monte Carlo Tree Search for workflow optimization. AGORA~\citep{zhang2025agora} unifies language agent algorithms with a graph-based orchestration engine, demonstrating that simpler methods like Chain-of-Thought often exhibit robust performance with significantly lower computational overhead than sophisticated graph approaches.

\paragraph{LangGraph.} LangGraph~\citep{langgraph} is currently the most widely adopted graph-based agent framework. It supports parallel node execution (via fan-out/fan-in patterns), conditional edges, state channels, and runtime graph modification through \texttt{add\_node}/\texttt{add\_edge} calls during execution. In our scheduler framework, LangGraph is a \emph{multi-ready-unit scheduler} with a semi-deterministic $\Policy$: the graph topology constrains which nodes may run in parallel, but the conditional-edge routing is determined by LLM output at runtime, making the effective policy non-deterministic.

The key differences between LangGraph and \SGH{} are philosophical rather than structural:

\begin{table}[t]
\centering
\caption{\SGH{} vs.\ LangGraph.}
\label{tab:vs-langgraph}
\begin{tabular}{@{}lp{3.2cm}p{3.2cm}@{}}
\toprule
\textbf{Dimension} & \textbf{LangGraph} & \textbf{\SGH{}} \\
\midrule
Plan mutability & Mutable at runtime (add/remove nodes and edges) & Immutable per plan version \\
Recovery & Retry + conditional routing & Three-level escalation protocol \\
Routing policy & LLM-driven conditional edges & Deterministic topology-based \\
Failure attribution & State-channel snapshots & Plan-version-audited traces \\
Expressiveness target & Maximum flexibility & Maximum controllability \\
Production maturity & Widely deployed; community-validated & Design only; no empirical validation \\
\bottomrule
\end{tabular}
\end{table}

\noindent\textit{Fairness disclaimer.} LangGraph is a mature, production-grade framework validated across thousands of deployments. \SGH{} is an unimplemented design. The comparison below is structural---intended to clarify design trade-offs---not evaluative. We do not claim that \SGH{} outperforms LangGraph; we claim that it occupies a different point in the design space.

LangGraph prioritizes \emph{flexibility}: developers can dynamically restructure the graph mid-execution, route based on LLM judgment, and combine arbitrary patterns. \SGH{} prioritizes \emph{controllability}: the plan-version immutability and escalation protocol trade away runtime flexibility for verifiable execution traces and bounded failure handling. These are complementary design points---LangGraph is the better tool for exploratory tasks where the workflow structure emerges during execution; \SGH{} is a candidate for engineering tasks where the dependency structure can be articulated upfront and verifiable execution matters. Whether this candidate delivers its promised benefits in practice remains an open empirical question.

A comprehensive survey by \citet{zhang2025gla} categorizes graph-augmented LLM agents by function---planning, memory, tool usage, and multi-agent coordination---and identifies key open challenges including structural adaptability and unified graph abstractions. \citet{yue2025survey} introduce Agentic Computation Graphs (ACGs) as a unifying abstraction, distinguishing static templates from dynamic runtime graphs and execution traces.

Our work differs from the above approaches in two key respects. First, we provide a \emph{scheduler-theoretic framework} that makes the trade-off between expressiveness and controllability \emph{formal}---not just an engineering observation but a consequence of the ready-set cardinality $|\Units|$. Second, we deliberately restrict expressiveness (no \texttt{first\_of}, no recursive expansion, no dynamic topology) to maximize controllability and verifiability, a design point that existing graph orchestration systems do not target.

\subsection{Multi-Agent Task Decomposition and Scheduling}

In multi-agent settings, task decomposition and scheduling have received significant attention. AOP~\citep{sun2025aop} proposes agent-oriented planning with three design principles---solvability, completeness, and non-redundancy---for multi-agent task allocation. DynTaskMAS~\citep{li2025dyntaskmas} introduces dynamic task graphs with asynchronous parallel execution, achieving 21--33\% reduction in execution time. Graph-of-Agents~\citep{liu2025graphofagents} models multi-agent communication as a directed graph with structured message passing, outperforming baselines using only 3 selected agents. G-Designer~\citep{zhang2025gdesigner} dynamically designs task-aware communication topologies using variational graph auto-encoders. WorfBench~\citep{qiao2025worfbench} benchmarks workflow generation, revealing a 15\% gap between sequence planning and graph planning capabilities even in GPT-4.

These systems demonstrate the power of graph-structured coordination but typically treat the graph topology as flexible and dynamically adjustable. \SGH{} takes the opposite stance: we argue that \emph{fixing} the topology for the duration of a plan version provides controllability guarantees that dynamic systems cannot offer, at the cost of excluding tasks that require runtime structural adaptation.

\subsection{Structured Execution for Scientific Agents}

El Agente Gr{\'a}fico~\citep{chen2025elagentegrafico} embeds LLM decision-making within type-safe execution graphs and knowledge graphs for scientific workflows, demonstrating that a single agent with a reliable execution engine can robustly perform complex multi-step computations. This validates our thesis that structured execution infrastructure matters more than agent count for complex tasks.

\subsection{Classical DAG Scheduling}

The formal study of DAG scheduling has a long history in parallel and distributed computing~\citep{topcuoglu2002list, cormen2009introduction}. List scheduling algorithms, topological ordering, and critical-path analysis provide the theoretical foundations upon which \SGH{} builds. Our contribution is not a new DAG scheduler but the identification that the same formal tools can be applied to analyze and improve LLM agent execution. While the applicability of DAG scheduling to workflow management may appear obvious in hindsight, the specific challenges introduced by non-deterministic LLM nodes have not been systematically addressed. \Cref{tab:classical-vs-llm} makes this gap explicit.

\begin{table}[t]
\centering
\caption{Classical DAG scheduling vs.\ LLM agent scheduling: why the analogy is not trivial.}
\label{tab:classical-vs-llm}
\begin{tabular}{@{}lp{3.5cm}p{3.5cm}@{}}
\toprule
\textbf{Classical concept} & \textbf{LLM agent challenge} & \textbf{\SGH{} mechanism} \\
\midrule
Deterministic node output & Same input may yield different outputs; hallucination & Contract validation + validation gap analysis \\
Infrastructure failures & Reasoning errors, hallucinations, misinterpretation & Three-level recovery distinguishing error types \\
Scheduling = optimization & Scheduling involves side-effect safety & Side-effect classification (Principle~4) \\
Context = task parameters & Context includes reasoning history that may corrupt subsequent steps & Execution/diagnostic context separation \\
Retry = idempotent & Retry may produce different results; non-idempotent side effects & Bounded retry + side-effect-aware scheduling \\
\bottomrule
\end{tabular}
\end{table}

Classical workflow engines (Airflow, Luigi, Prefect) already apply DAG scheduling to task orchestration, as we discuss in \Cref{sec:workflow-engines}. Our contribution relative to both classical scheduling and classical workflow engines is the identification and systematic treatment of the five challenges in \Cref{tab:classical-vs-llm} that arise specifically when nodes are powered by non-deterministic LLM inference.

\paragraph{Theoretical connections.} The topological scheduling policy used by \SGH{} (dispatching all ready nodes in topological order) is equivalent to \textit{list scheduling} with zero communication costs, a classical algorithm that achieves a $2 - 1/m$ approximation ratio for $m$ identical processors~\citep{graham1969bounds}. This connection suggests that, for tasks where LLM execution time is approximately constant, \SGH{}'s default policy provides provable performance guarantees. For LLM nodes with large context windows and variable execution times, results on DAG scheduling with communication costs~\citep{sinnen2007new} may inform optimal policy design. Online scheduling with unknown job durations is also relevant, as LLM execution time is highly variable and difficult to predict in advance. These classical results provide a rich theoretical foundation for future extensions of \SGH{}'s scheduling policy.

\subsection{Workflow Engines: Airflow, Luigi, and Prefect}
\label{sec:workflow-engines}

Classic DAG-based workflow engines---Apache Airflow, Luigi, Prefect, and their descendants---share several surface-level features with \SGH{}: static DAG definitions, topological scheduling, and retry mechanisms. A natural question is: what is genuinely new?

\begin{table}[t]
\centering
\caption{\SGH{} vs.\ classical workflow engines.}
\label{tab:vs-workflow}
\begin{tabular}{@{}lp{3.5cm}p{3.5cm}@{}}
\toprule
\textbf{Dimension} & \textbf{Airflow / Luigi / Prefect} & \textbf{\SGH{}} \\
\midrule
Node behavior & Deterministic (Python code) & Non-deterministic (LLM inference) \\
Node output & Predictable, type-checked & Unpredictable, contract-validated \\
Failure causes & Primarily infrastructural & Primarily reasoning errors, hallucinations \\
Recovery & Manual / simple retry & Three-level escalation protocol \\
Context management & Task-scoped variables & Execution/diagnostic context separation \\
Plan evolution & Redeploy new DAG version & Plan-version immutability within execution \\
DAG dynamics & Dynamic at parse-time (Python-generated) & Static per plan-version \\
\bottomrule
\end{tabular}
\end{table}

The fundamental difference is that \SGH{}'s nodes are \emph{non-deterministic}: the same node with the same inputs may produce different outputs on different invocations, because the computation is performed by an LLM. This non-determinism creates challenges that classical workflow engines do not face: output contracts cannot be checked at compile time, failure modes include hallucination and misinterpretation (not just crashes), and the recovery strategy must account for the possibility that the LLM's reasoning---not the infrastructure---is the source of the error.

\SGH{} borrows the formal structure of DAG scheduling (topological ordering, dependency tracking, ready-set computation) but adds three mechanisms specifically designed for non-deterministic nodes: (1)~contract-based output validation with explicit pass/fail semantics; (2)~a three-level recovery protocol that distinguishes transient errors from reasoning failures; and (3)~execution/diagnostic context separation that prevents failure history from corrupting subsequent reasoning. These additions have no analogue in classical workflow engines and constitute the core of \SGH{}'s contribution.

\subsection{Systematic Comparison}
\label{sec:systematic-comparison}

To clarify \SGH{}'s positioning relative to existing work, \Cref{tab:systematic-comparison} provides a systematic comparison along seven key dimensions. This table makes explicit the design trade-offs that distinguish \SGH{} from related systems.

\begin{table}[t]
\centering
\caption{Systematic comparison of \SGH{} with related LLM agent and workflow systems.}
\label{tab:systematic-comparison}
\tiny
\begin{tabular}{@{}lccccccc@{}}
\toprule
\textbf{System} & \textbf{Multi-} & \textbf{Det.} & \textbf{Bounded} & \textbf{Immut.} & \textbf{Recovery} & \textbf{Contract} & \textbf{Target} \\
& \textbf{Ready} & \textbf{Policy} & \textbf{Recovery} & \textbf{Plan} & \textbf{Protocol} & \textbf{Valid.} & \textbf{Domain} \\
\midrule
Agent Loop~\citep{yao2023react} & No & No & No & No & Ad-hoc LLM & None & General \\
AutoGPT & No & No & No & No & Ad-hoc LLM & None & General \\
CrewAI & No & No & No & No & Ad-hoc LLM & None & General \\
Plan-and-Act~\citep{erdogan2025planandact} & No & No & No & Partial & Re-plan & Optional & Web nav \\
TDP~\citep{wang2025tdp} & No & No & No & Mutable & Node-local & Optional & General \\
LangGraph~\citep{langgraph} & \textbf{Yes} & \textbf{Semi} & \textbf{No} & \textbf{Mut.} & \textbf{Retry} & \textbf{Opt.} & \textbf{Gen.} \\
GPTSwarm~\citep{zhuge2024gptswarm} & \textbf{Yes} & \textbf{No} & \textbf{No} & \textbf{Mut.} & \textbf{Edge recon.} & \textbf{Opt.} & \textbf{Gen.} \\
DynTaskMAS~\citep{li2025dyntaskmas} & \textbf{Yes} & \textbf{No} & \textbf{No} & \textbf{Dyn.} & \textbf{Dyn. resch.} & \textbf{Opt.} & \textbf{Gen.} \\
Airflow & \textbf{Yes} & \textbf{Yes} & \textbf{Yes} & \textbf{Yes} & \textbf{Manual} & N/A & Workflow \\
Luigi & \textbf{Yes} & \textbf{Yes} & \textbf{Yes} & \textbf{Yes} & \textbf{Manual} & N/A & Workflow \\
Prefect & \textbf{Yes} & \textbf{Yes} & \textbf{Yes} & \textbf{Yes} & \textbf{Manual} & N/A & Workflow \\
\midrule
\textbf{\SGH{}} & \textbf{Yes} & \textbf{Yes} & \textbf{Yes} & \textbf{Yes} & \textbf{Escalation} & \textbf{Req.} & \textbf{Eng.} \\
\bottomrule
\end{tabular}
\end{table}

\paragraph{Key observations.} \Cref{tab:systematic-comparison} reveals three patterns. First, \SGH{} is the \emph{only} LLM agent system that simultaneously enforces all four core constraints: multi-ready-unit scheduling, deterministic policy, bounded recovery, and immutable plan versions. LangGraph achieves multi-ready-unit scheduling but lacks the other three. TDP achieves plan-structure but remains single-ready-unit. Second, classical workflow engines (Airflow, Luigi, Prefect) meet these constraints, but their nodes are deterministic (not LLM-powered) and they lack contract validation and LLM-specific recovery. Third, \SGH{} is explicitly positioned for \emph{engineering tasks} with verifiable outcomes, whereas most LLM agent systems target \emph{general-purpose} tasks with open-ended exploration.

\paragraph{LLM-specific challenges addressed.} Classical workflow engines do not face three challenges that \SGH{} explicitly addresses: (1)~\textbf{Non-deterministic node output}: the same LLM node may produce different outputs on different invocations, requiring contract validation rather than type-checking. (2)~\textbf{Reasoning failures as primary error mode}: LLM nodes fail due to hallucinations or misinterpretation, not just infrastructure crashes. (3)~\textbf{Non-idempotent retry}: retrying an LLM call may produce different results, requiring side-effect classification and bounded budgets. \Cref{tab:classical-vs-llm} (Section 8.8) makes this gap explicit.

\subsection{Positioning}

\begin{table}[t]
\centering
\caption{Positioning \SGH{} relative to existing approaches.}
\label{tab:positioning}
\begin{tabular}{@{}lp{2.5cm}p{2.5cm}p{2cm}p{2cm}@{}}
\toprule
\textbf{Approach} & \textbf{Execution commitment} & \textbf{Recovery} & \textbf{Scheduler type} & \textbf{LLM-aware?} \\
\midrule
Agent Loop~\citep{yao2023react} & None & Unbounded LLM decision & Single-ready-unit & Implicit \\
Plan-and-Act~\citep{erdogan2025planandact} & Plan as prompt & Re-plan & Single-ready-unit & Partial \\
TDP~\citep{wang2025tdp} & Mutable DAG sub-tasks & Node-local replan & Single-ready-unit & Partial \\
LangGraph~\citep{langgraph} & Mutable runtime graph & Retry + conditional routing & Multi-ready-unit & Partial \\
GPTSwarm~\citep{zhuge2024gptswarm} & Optimizable graph & Edge reconnection & Multi-ready-unit & No \\
DynTaskMAS~\citep{li2025dyntaskmas} & Dynamic DAG & Dynamic reschedule & Multi-ready-unit & No \\
Airflow & Immutable DAG version & Manual / simple retry & Multi-ready-unit & No \\
Luigi & Immutable DAG version & Manual / simple retry & Multi-ready-unit & No \\
Prefect & Immutable DAG version & Manual / simple retry & Multi-ready-unit & No \\
\midrule
\textbf{\SGH{}} & \textbf{Immutable plan version} & \textbf{Escalation protocol} & \textbf{Multi-ready-unit} & \textbf{Yes (full)} \\
\bottomrule
\end{tabular}
\end{table}

\SGH{} is a design that, unlike prior work, simultaneously enforces three properties as mandatory constraints: (1)~multi-ready-unit scheduling with a deterministic policy, (2)~immutable plan versions as execution commitments, and (3)~a bounded recovery protocol with escalation invariants. LangGraph achieves multi-ready-unit scheduling but does not guarantee plan immutability or bounded recovery. TDP achieves DAG-structured decomposition with scoped contexts but remains single-ready-unit and lacks escalation invariants. Each of \SGH{}'s three properties exists in isolation in prior work; their combination and the identification of the design tensions that arise from enforcing them jointly are the contributions.

\section{A Scheduler-Unified Framework}
\label{sec:scheduler-framework}

This section develops a formal framework that places Agent Loops and graph-based executors on the same semantic footing. The key insight is that both systems can be modeled as \emph{schedulers} that differ along two dimensions: the cardinality of the ready set and the explicitness of the scheduling policy.

\subsection{Execution Systems}

\begin{definition}[Execution System]
\label{def:execution-system}
An \emph{execution system} is a tuple
\[
  \ExecSystem = (\State, \Units, \Policy, \Obs, \Update)
\]
where
\begin{itemize}[leftmargin=*]
  \item $\State = \{(v, s_v) \mid v \in V, s_v \in \Sigma\}$ is the set of node states, where $V$ is the set of nodes (executable units) and $\Sigma$ is the node state set (see Definition~\ref{def:node-state});
  \item $\Units : \State \to 2^V$ maps a global state to the ready set of nodes eligible for dispatch:
        \[ \Units(\State) = \{v \in V \mid s_v = \mathit{ready} \land \forall (u, v) \in E: s_u = \mathit{executed}\} \]
        where $E$ is the edge set representing dependencies. A node becomes ready when all its predecessors have executed;
  \item $\Policy : 2^V \to V$ is the \emph{scheduling policy}---a deterministic function that selects a single node from the ready set. In non-deterministic systems, $\Policy$ is a relation rather than a function;
  \item $\Obs = \{\mathit{success}, \mathit{failure}, \mathit{retry}, \mathit{escalate}\}$ is the \emph{outcome space} for node execution;
  \item $\Update : \State \times V \times \Obs \to \State$ is the \emph{state transition function} that updates the node's state based on the execution outcome and recomputes $\Units(\State')$ for the new global state.
\end{itemize}
\end{definition}

\begin{example}[Transition Example]
  Given $\ExecSystem$ with node $v$ in state $\mathit{running}$:
  \begin{itemize}
  \item If observation $\mathit{success}$ occurs, $\Update(\State, v, \mathit{success})$ sets $s_v \leftarrow \mathit{executed}$ and recomputes $\Units(\State')$, potentially enabling successor nodes;
  \item If observation $\mathit{failure}$ occurs, $\Update(\State, v, \mathit{failure})$ sets $s_v \leftarrow \mathit{failed}$ and may trigger recovery actions via the recovery layer.
  \end{itemize}
\end{example}

Using a relation rather than a function for $\Policy$ is essential: it lets the framework encompass both deterministic schedulers (where the next unit is uniquely determined by the state) and non-deterministic ones (where the same state may yield different choices on different invocations). As we show below, the Agent Loop falls into the latter category.

\subsection{Single-Ready-Unit Schedulers}

\begin{definition}[Single-Ready-Unit Scheduler]
\label{def:single-ready}
An execution system $\ExecSystem$ is a \emph{single-ready-unit scheduler} if, at every reachable state $s \in \State$, the ready set satisfies $|\Units(s)| \le 1$.
\end{definition}

When $|\Units| \le 1$, the scheduling relation $\Policy$ becomes vacuous: there is at most one unit to choose, so no meaningful scheduling decision exists. The ``choice'' of what to do next is effectively determined by whatever process produces the next element of $\Units$---in the Agent Loop's case, an LLM inference.

\paragraph{Observation.}
\label{obs:agent-loop-single}
The Agent Loop is a non-deterministic, single-ready-unit scheduler.

This follows directly from the Agent Loop's design: at each step, the LLM produces exactly one action $a$ (so $|\Units| = 1$), but the same context state may yield different actions on different invocations (so $\Policy$ is not functional).

\paragraph{Parallel tool calls.} Some Agent Loop implementations support parallel tool invocation within a single LLM call (e.g., OpenAI's \texttt{tool\_calls} field can contain multiple function calls that execute concurrently). In our framework, this corresponds to $|\Units| > 1$ but with a \emph{non-deterministic} $\Policy$---the set of parallel operations is determined by the LLM, not by an explicit scheduling policy. This distinction is crucial: even with parallel tool calls, the Agent Loop lacks \emph{structural} parallelism guarantees. The LLM may choose to invoke tools in parallel on one turn but sequentially on the next, depending on its internal reasoning. \SGH{}'s advantage is not merely parallelism, but \emph{explicit} parallelism: the ready set is computed from the DAG topology, and parallel dispatch is guaranteed whenever dependencies allow.

\paragraph{Boundary cases: Parallel Loop with Planner.} Some Agent Loop systems support both parallel tool calls and explicit planning. In our classification framework, such systems are classified based on their \emph{execution layer} behavior. If the execution remains a single LLM call per step (even with parallel tool calls), the system is classified as "Parallel Loop" ($|U| \geq 1$, non-deterministic Policy). If the planning produces a multi-step structure that is executed with a deterministic scheduler, the system moves toward the multi-ready-unit regime. The key distinction is whether the ready set is \emph{explicitly computed from a static structure} (\SGH{}) or \emph{determined dynamically by LLM inference} (Agent Loop variants).

\paragraph{Asynchronous operations.} Some Agent Loops support asynchronous operations (e.g., initiating multiple file reads in parallel, then waiting for all results). In our framework, this also corresponds to $|\Units| > 1$ with a non-deterministic policy. However, this form of parallelism is \emph{ad-hoc}---the LLM decides which operations to parallelize and when. \SGH{}'s explicit DAG makes parallelism a structural property: if node A and node B have no dependencies between them, they will be dispatched in parallel regardless of LLM inference.

\subsection{Multi-Ready-Unit Schedulers}

\begin{definition}[Multi-Ready-Unit Scheduler]
\label{def:multi-ready}
An execution system $\ExecSystem$ is a \emph{multi-ready-unit scheduler} if there exists a reachable state $s \in \State$ with $|\Units(s)| > 1$.
\end{definition}

When $|\Units| > 1$, the scheduling relation becomes a genuine design decision. Different policies (topological order, priority-based, parallel dispatch) produce different execution traces from the same ready set. This is the regime in which graph-based executors operate.

\subsection{The Scheduler Continuum}

\paragraph{Framework observation (Scheduler Continuum).}
\label{obs:continuum}
Execution systems for LLM agents can be arranged along a continuum parameterized by three axes. First is \textbf{ready-set cardinality}: $|\Units| = 1$ (single) versus $|\Units| \ge 1$ (multi). Second is \textbf{policy explicitness}: implicit (LLM inference over context) versus prompt-level (structured prompts constrain LLM output) versus state-machine (explicit policy over formal state). Third is \textbf{policy determinism}: non-deterministic ($\Policy$ is a relation) versus deterministic ($\Policy$ is a function).
This is a \emph{classification framework}, not a theorem. Its value is analytical: it makes precise what varies across systems and generates theoretical predictions (see below).

\paragraph{Theoretical predictions.} The scheduler continuum generates three theoretical predictions that, \emph{in principle}, can be tested using the experimental framework of \Cref{sec:experiment-framework}. First, $G_{\mathit{graph}} > 0$: moving from single-ready-unit to multi-ready-unit scheduling yields a measurable performance benefit, independent of planning quality. Second, $G_{\mathit{graph}}$ increases with task complexity: the benefit of multi-ready-unit scheduling grows as dependency depth and branching factor increase. Third, $G_{\mathit{replan}} > 0$ specifically on failure-prone tasks: the escalation protocol provides measurable benefit only when tasks exhibit non-trivial failure rates.

\textit{Important caveat.} These predictions are \emph{not empirically validated} in this paper, intended to guide future empirical work. They are logical consequences of the framework's assumptions and the scheduler-theoretic model. While plausible, they may not hold in practice. Their falsifiability is theoretical: the experimental framework provides a methodology for testing them, but executing those experiments requires a prototype implementation and curated benchmarks---both of which are ongoing work. We present these predictions as guidance for future research, not as proven claims.

\Cref{fig:continuum} illustrates the continuum along the first two axes. Moving from left to right, the ready set grows and the scheduling policy becomes more explicit. Importantly, ``adding a planner to the loop'' moves the system along the policy-explicitness axis but does not change the ready-set cardinality. Only a graph-based executor can reach the multi-ready-unit regime.

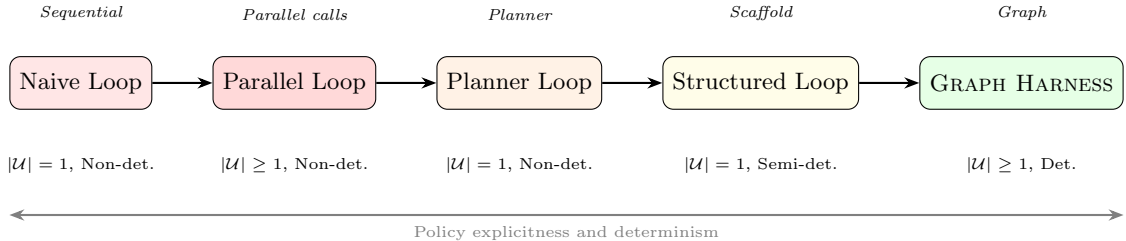
\begin{figure}[t]
\centering
\begin{tikzpicture}[
  node distance=0.8cm,
  box/.style={draw, rounded corners, minimum width=1.3cm, minimum height=0.7cm, align=center, font=\footnotesize},
  arr/.style={-{Stealth[length=2mm]}, thick}
]
  \node[box, fill=red!10] (nl) {Naive Loop};
  \node[box, fill=red!15, right=of nl] (pl) {Parallel Loop};
  \node[box, fill=orange!10, right=of pl] (pl2) {Planner Loop};
  \node[box, fill=yellow!10, right=of pl2] (sl) {Structured Loop};
  \node[box, fill=green!10, right=of sl] (sg) {\SGH{}};

  \draw[arr] (nl) -- (pl);
  \draw[arr] (pl) -- (pl2);
  \draw[arr] (pl2) -- (sl);
  \draw[arr] (sl) -- (sg);

  \node[below=0.5cm of nl, font=\tiny] {$|\Units|=1$, Non-det.};
  \node[below=0.5cm of pl, font=\tiny] {$|\Units|\ge1$, Non-det.};
  \node[below=0.5cm of pl2, font=\tiny] {$|\Units|=1$, Non-det.};
  \node[below=0.5cm of sl, font=\tiny] {$|\Units|=1$, Semi-det.};
  \node[below=0.5cm of sg, font=\tiny] {$|\Units|\ge1$, Det.};

  \node[above=0.35cm of nl, font=\tiny\itshape] {Sequential};
  \node[above=0.35cm of pl, font=\tiny\itshape] {Parallel calls};
  \node[above=0.35cm of pl2, font=\tiny\itshape] {Planner};
  \node[above=0.35cm of sl, font=\tiny\itshape] {Scaffold};
  \node[above=0.35cm of sg, font=\tiny\itshape] {Graph};

  \draw[{Stealth[length=2mm]}-{Stealth[length=2mm]}, thick, gray]
    ([yshift=-1.4cm]nl.south west) -- ([yshift=-1.4cm]sg.south east)
    node[midway, below, font=\tiny] {Policy explicitness and determinism};
\end{tikzpicture}
\caption{The scheduler continuum (extended). The new ``Parallel Loop'' category represents Agent Loops with parallel tool calls (e.g., OpenAI's \texttt{tool\_calls} field). These achieve $|\Units| \ge 1$ but with a non-deterministic policy, distinguishing them from \SGH{}'s explicit, topology-based scheduling.}
\label{fig:continuum}
\end{figure}

\subsection{Why Unification Matters}

The scheduler framework provides three analytical benefits:

\paragraph{Comparability.} Agent Loops and graph executors can be compared on the same axes. The question ``is a graph executor better than an Agent Loop?'' becomes ``does moving from $|\Units|=1$ to $|\Units| \ge 1$ with a deterministic $\Policy$ improve performance, and by how much?''

\paragraph{Expressiveness characterization.} Expressiveness is precisely the range of ready-set configurations an execution system can produce. A single-ready-unit scheduler is expressiveness-bounded by definition: it cannot represent concurrent, conditional, or fallback execution.

\paragraph{Controllability characterization.} Controllability is precisely the degree to which $\Policy$ is functional and $\State$ is explicit. An Agent Loop scores low on both counts; \SGH{} scores high by construction.

\paragraph{Derivable hypotheses.} The framework generates three hypotheses that can, in principle, be tested with the experimental framework of \Cref{sec:experiment-framework}: (1)~$G_{\mathit{graph}} > 0$, i.e., moving to multi-ready-unit scheduling should yield a measurable benefit independent of planning; (2)~$G_{\mathit{graph}}$ should increase with task complexity; (3)~$G_{\mathit{recovery}} > 0$ primarily on tasks with non-trivial failure rates. Validating these hypotheses is left to future work; here we simply note that they are logical consequences of the framework's assumptions.

\paragraph{Computational complexity.} The scheduler framework provides a **semantic** classification of execution systems but does not yet provide **complexity-theoretic** guarantees. For example, we do not prove that computing the optimal schedule under budget constraints is NP-hard (though this is likely, given the reduction from classical DAG scheduling~\citep{topcuoglu2002list}). Formalizing the computational complexity of SGH scheduling is a direction for future work.

\subsection{Worked Example}
\label{sec:worked-example}

We now trace a concrete task through both an Agent Loop and \SGH{} to illustrate how the formal definitions translate into concrete behavior and where the structural differences manifest.

\paragraph{Task.} Investigate and fix a Python authentication bug, while simultaneously updating documentation. The task decomposes into ten steps:

\begin{enumerate}[leftmargin=*]
  \item \textbf{Search auth} (\texttt{search\_auth}): Find authentication-related files.
  \item \textbf{Search utils} (\texttt{search\_utils}): Find utility/helper files.
  \item \textbf{Read auth} (\texttt{read\_auth}): Read the auth source file. \emph{Depends on: \texttt{search\_auth}.}
  \item \textbf{Read utils} (\texttt{read\_utils}): Read the utils source file. \emph{Depends on: \texttt{search\_utils}.}
  \item \textbf{Analyze} (\texttt{analyze}): Identify the root cause. \emph{Depends on: \texttt{read\_auth}, \texttt{read\_utils}} (all\_of).
  \item \textbf{Fix A} (\texttt{fix\_A}): Apply patch to auth module. \emph{Depends on: \texttt{analyze}.}
  \item \textbf{Fix B} (\texttt{fix\_B}): Apply alternative patch (different approach). \emph{Depends on: \texttt{analyze}.}
  \item \textbf{Test} (\texttt{run\_tests}): Run the test suite. \emph{Depends on: \texttt{fix\_A} or \texttt{fix\_B}} (any\_of).
  \item \textbf{Docs} (\texttt{update\_docs}): Update API documentation. \emph{Depends on: \texttt{analyze}.}
  \item \textbf{Report} (\texttt{report}): Generate summary. \emph{Depends on: \texttt{run\_tests}, \texttt{update\_docs}} (all\_of).
\end{enumerate}

This DAG has three notable structural features (\Cref{fig:running-example-dag}): (1)~\texttt{search\_auth} and \texttt{search\_utils} are independent and can execute in parallel ($|\Units| = 2$); (2)~\texttt{fix\_A} and \texttt{fix\_B} are alternative patches, only one of which needs to succeed (any\_of join); (3)~\texttt{update\_docs} and the fix/test path can proceed in parallel.

\subsubsection{Agent Loop Execution}

In the Agent Loop, the LLM reasons over a growing context window:

\begin{verbatim}
Turn 1:  LLM decides to search auth files
          → search_code("auth bug")
Turn 2:  LLM reads results, decides to search utils too
          → search_code("utils helper")    [sequential, not parallel]
Turn 3:  LLM reads auth file
          → read_file("auth.py")
Turn 4:  LLM reads utils file
          → read_file("utils.py")
Turn 5:  LLM analyzes both files
          → analyze(...)                    [implicit dep on Turns 3-4]
Turn 6:  LLM writes a fix
          → write_fix("patch A")            [no alternative tried]
Turn 7:  LLM runs tests
          → run_tests(...)
          → FAIL: 2 tests still failing
Turn 8:  LLM decides to try a different fix
          → write_fix("patch B")            [ad-hoc, no protocol]
Turn 9:  LLM runs tests again
          → run_tests(...)
          → OK: all pass
Turn 10: LLM updates docs
          → update_docs(...)
Turn 11: LLM generates report
\end{verbatim}

\paragraph{Problems illustrated.} The Agent Loop execution illustrates several structural weaknesses. First, there is no parallelism: \texttt{search\_auth} and \texttt{search\_utils} could run concurrently, but the Agent Loop processes them sequentially (Turns 1--2). Similarly, \texttt{update\_docs} could run in parallel with the fix/test cycle, but the Agent Loop does them serially (Turns 8--10). Second, there are no structured alternatives: the LLM tries one fix first, then switches to another after failure, with no mechanism to express ``try A or B, whichever works.'' The switch is an ad-hoc LLM decision. Third, recovery is unbounded: at Turn 7, the LLM autonomously decides to try a different approach with no bound on how many alternatives it might attempt, no formal diagnosis of why the first fix failed, and no protocol for escalation. Fourth, dependencies are implicit: turns 3--5 depend on prior results through context, not through structural guards, and a long context window increases the risk of ``forgetting'' that \texttt{analyze} requires \emph{both} files. Finally, there is no plan versioning: the original plan (``fix the auth module'') was implicitly revised mid-execution (``try a different approach''), with no audit trail recording this change.

\subsubsection{\SGH{} Execution}

The planner generates a static DAG (\Cref{fig:running-example-dag}) before execution begins.

\begin{verbatim}
Plan version 1 generated. DAG validated (acyclic, reachable, joins consistent).

Ready-set: {search_auth, search_utils}     |U| = 2
  → dispatch search_auth                   → executed
  → dispatch search_utils                  → executed    [parallel]
Ready-set: {read_auth, read_utils}         |U| = 2
  → dispatch read_auth                     → executed
  → dispatch read_utils                    → executed    [parallel]
Ready-set: {analyze}                       |U| = 1
  → dispatch analyze                       → executed
Ready-set: {fix_A, fix_B, update_docs}     |U| = 3
  → dispatch fix_A                         → FAILED (transient)
  → dispatch fix_B                         → executed
  → dispatch update_docs                   → executed    [parallel]
  fix_A: Level 1 recovery → skipped (fix_B already succeeded for any_of)
Ready-set: {run_tests}                     |U| = 1
  → dispatch run_tests                     → executed
Ready-set: {report}                        |U| = 1
  → dispatch report                        → executed

All nodes terminal. Plan version 1 complete.
\end{verbatim}

\paragraph{Structural differences.} The SGH execution demonstrates five key structural differences. First, parallel execution: three ready-set expansions produce $|\Units| > 1$, with searches running in parallel ($|\Units|=2$), reads running in parallel ($|\Units|=2$), and fix/docs running concurrently ($|\Units|=3$). The Agent Loop cannot achieve this because $|\Units| \le 1$ by design. Second, structured alternatives: the any\_of join from \texttt{fix\_A}/\texttt{fix\_B} to \texttt{run\_tests} encodes ``try either patch'' as a first-class structural feature, not an ad-hoc LLM decision, so when \texttt{fix\_B} succeeds, \texttt{fix\_A} is skipped with no wasted retry. Third, bounded recovery: \texttt{fix\_A}'s transient failure triggers Level~1 recovery, but the any\_of join makes recovery unnecessary because \texttt{fix\_B} already satisfies the dependency, so the system does not waste budget retrying a failed path when an alternative has already succeeded. Fourth, explicit dependencies: each node's ready-set membership is computed from the DAG, so \texttt{analyze} waits for \emph{both} \texttt{read\_auth} and \texttt{read\_utils}---this is enforced structurally, not by LLM memory. Fifth, auditability: every state transition is recorded with its plan version, and the audit trail shows that \texttt{fix\_A} failed, \texttt{fix\_B} succeeded, and \texttt{fix\_A} was skipped (not retried wastefully).

\paragraph{Step count comparison.} The Agent Loop required 11 turns (all sequential). \SGH{} required 10 node dispatches across 6 scheduling rounds, with 4 of those rounds dispatching multiple nodes in parallel. Wall-clock time is reduced proportionally to the degree of available parallelism.

\subsection{Planning Failures and Their Consequences}
\label{sec:planning-failures}

The worked example in \Cref{sec:worked-example} assumes the planner generates a perfect DAG---one that correctly identifies all dependencies, all parallelism opportunities, and the correct join semantics (all\_of vs any\_of). In practice, planning can fail in several ways, each with different consequences.

\paragraph{Missing dependencies.} The planner may forget to add an edge between nodes that actually depend on each other. For example, it might omit the dependency from \texttt{read\_auth} to \texttt{analyze}, allowing \texttt{analyze} to execute before \texttt{read\_auth} completes. In \SGH{}, this manifests as a contract violation: \texttt{analyze} receives incomplete input (missing \texttt{read\_auth} results) and fails the output contract validation, triggering the recovery protocol. The three-level escalation ensures the system does not silently proceed with corrupt data.

\paragraph{Spurious dependencies.} The planner may add unnecessary dependencies, restricting parallelism. For example, it might incorrectly assert that \texttt{search\_utils} must wait for \texttt{search\_auth} to complete, even though they are independent. This reduces $|\Units|$ from 2 to 1, eliminating parallelism. \SGH{} does not detect this error---the DAG is structurally valid---but the execution is suboptimal. The cost is wasted time (sequential execution instead of parallel), not correctness.

\paragraph{Incorrect branch selection.} The planner may choose the wrong join semantics. For example, it might use \texttt{all\_of} instead of \texttt{any\_of} for the patch selection, requiring both \texttt{fix\_A} and \texttt{fix\_B} to succeed before proceeding. This is particularly damaging when only one patch is viable; the system will repeatedly retry the failing patch (level 1 recovery), escalate to patch generation (level 2 recovery), and eventually require a full replan (level 3 recovery). The cost is wasted execution time and tokens.

\paragraph{Over-decomposition.} The planner may split a task into too many fine-grained nodes, increasing DAG size and overhead. For example, splitting \texttt{read\_auth} into \texttt{find\_auth\_file}, \texttt{open\_auth\_file}, and \texttt{read\_auth\_lines} introduces unnecessary serialization and coordination overhead. \SGH{} executes each node separately, so the overhead scales with node count. There is no automatic mechanism to merge nodes; the planner must choose appropriate granularity.

\paragraph{Under-decomposition.} The planner may lump multiple operations into a single node, missing parallelism opportunities. For example, combining \texttt{search\_auth} and \texttt{search\_utils} into a single \texttt{search\_files} node eliminates parallelism and makes error attribution difficult (if \texttt{search\_files} fails, it's unclear which search failed). \SGH{} cannot recover the lost parallelism without a replan.

\paragraph{Summary.} Planning failures are inevitable in practice. \SGH{} handles them through three mechanisms: (1)~\textbf{Structural validation} catches cycles, unreachable nodes, and inconsistent joins. (2)~\textbf{Contract validation} catches missing dependencies and incorrect inputs at runtime. (3)~\textbf{Recovery protocol} handles execution failures and can trigger replans when necessary. However, \SGH{} cannot fix spurious dependencies, incorrect branch selection, over-decomposition, or under-decomposition without a new plan version. The trade-off between planning accuracy and recovery cost is an open question that the experimental framework (\Cref{sec:experiment-framework}) is designed to quantify through the $G_{\mathit{plan}}$ metric.

\subsubsection{Scheduler-Framework Mapping}

The worked example illustrates how the formal definitions (\Cref{sec:scheduler-framework}) translate into concrete behavior. The DAG's ten nodes form the action universe $\mathcal{A}$, with each step (search, read, analyze, fix, test, docs, report) representing an executable unit. The execution state $\State$ records each node's status (ready, running, executed, failed, cancelled), which is defined in \Cref{def:node-state}. The ready set $\Units(\State)$ computes which nodes have all predecessors executed; for example, after the two searches complete, $\Units = \{\texttt{read\_auth}, \texttt{read\_utils}\}$ ($|\Units| = 2$), which is the core of \Cref{def:execution-system}. The scheduling policy $\Policy$ is deterministic, dispatching all ready nodes, which is why SGH achieves $|\Units| > 1$, whereas in the Agent Loop, $\Policy$ is non-deterministic and only ever produces $|\Units| = 1$. The observation function $\Obs$ records each node's outcome (success/failure/retry/escalate), and the observation of \texttt{fix\_A} as \textit{failure} and \texttt{fix\_B} as \textit{success} demonstrates this outcome space. Finally, the state transition function $\Update$ recomputes $\Units$ after each observation; when \texttt{fix\_B} succeeds, $\Update$ cancels \texttt{fix\_A} (any\_of join satisfied) and adds \texttt{run\_tests} to $\Units$.

\paragraph{Key properties illustrated.} This worked example demonstrates four properties that an Agent Loop ($|\Units| \le 1$) cannot exhibit. First, parallel dispatch: the ready set reaches cardinalities of 2 and 3, enabling concurrent execution of independent nodes, which is the structural advantage of multi-ready-unit scheduling. Second, structured alternatives: the any\_of join from \texttt{fix\_A}/\texttt{fix\_B} to \texttt{run\_tests} makes ``try either patch'' a structural property, not an LLM decision. Third, efficient recovery: when \texttt{fix\_A} fails but \texttt{fix\_B} succeeds, the system skips recovery entirely because the any\_of join is already satisfied, so no retry budget is wasted. Fourth, deterministic progress: the ready set at each step is fully determined by the DAG topology and node states, with no dependence on LLM inference for scheduling decisions.

\section{Design Principles and Trade-offs}
\label{sec:design-principles}

Any LLM agent execution system must balance three competing goals: \emph{controllability} (predictability, verifiability, auditability), \emph{expressiveness} (the range of task structures it can represent), and \emph{implementability} (engineering complexity and risk). This section analyzes the trade-off space and derives four design principles that govern \SGH{}.

\subsection{The Trade-off Space}

Our survey of 70 open-source projects (see \Cref{sec:appendix-survey} for methodology) shows a clear pattern (\Cref{tab:tradeoff-survey}). Projects are classified by primary execution pattern: \emph{Agent Loop} (single LLM call per step), \emph{Event-driven} (execution triggered by external events), \emph{State-machine} (explicit FSM governs transitions), \emph{Graph/flow} (nodes and edges form explicit topology), and \emph{Hybrid} (combines multiple patterns). In the scheduler framework, these map as follows: Agent Loop $\to$ $|\Units| \le 1$, non-deterministic $\Policy$; Event-driven $\to$ $|\Units| \le 1$, externally triggered $\Delta$; State-machine $\to$ $|\Units| \le 1$, semi-explicit $\Policy$; Graph/flow $\to$ $|\Units| \ge 1$, explicit $\Policy$.

\begin{table}[t]
\centering
\caption{Observed trade-offs across 70 surveyed agent systems.\label{tab:tradeoff-survey-note}Categories are mutually exclusive based on the primary execution pattern observed. Percentages are approximate; sensitivity analysis within $\pm$10\% does not change the ordinal ranking (Agent Loop is most common, Graph/flow is least common). See \Cref{sec:appendix-survey} for methodology.}
\label{tab:tradeoff-survey}
\begin{tabular}{@{}lccc@{}}
\toprule
\textbf{Category} & \textbf{Expressiveness} & \textbf{Controllability} & \textbf{Implementability} \\
\midrule
Agent Loop (60\%)       & Low    & Low    & High \\
Event-driven (15\%)     & Low    & Medium & High \\
State-machine (10\%)    & Medium & High   & Medium \\
Graph/flow orchestration (5\%) & \textbf{High} & Low & Low \\
Hybrid (10\%)           & Medium & Medium & Medium \\
\bottomrule
\end{tabular}
\end{table}

Systems with the highest expressiveness---graph and flow orchestration frameworks such as those built on LangGraph or custom FSM engines---consistently exhibit the lowest controllability and the highest implementation risk. Conversely, simple Agent Loops are trivially implementable but offer no structural guarantees.

\SGH{} targets a design point that is deliberately \emph{not} at the expressiveness frontier. The goal is to maximize controllability while retaining enough expressiveness for the broad class of \emph{verifiable engineering tasks}---tasks whose dependency structure can be articulated before execution begins.

\subsection{Principle 1: Controllability First}

\begin{quote}
\textit{In the face of uncertain benefit, prioritize the predictability and verifiability of the execution process.}
\end{quote}

\paragraph{Rationale from survey.} Our qualitative analysis of 70 agent systems showed an observed pattern: failure-loop behavior (infinite or unbounded recovery cycles) was frequently observed among the graph/flow orchestration systems in our dataset (3 out of 4 projects), while such behavior was rarely observed in state-machine-based systems (0 out of 7 projects). We did not formally quantify this observation due to the subjective nature of classifying "failure-loop behavior," and different reviewers might reach different conclusions. However, the qualitative pattern was consistent across multiple codebase inspections and GitHub issue analyses. This observation suggests that flexible, unstructured control flows---while expressive---may be more prone to recovery pathologies. For engineering tasks, the cost of an uncontrolled execution (silent data corruption, unreproducible failures, un-auditable decisions) often exceeds the benefit of representing a more complex task structure. Controllability is the foundation on which verifiability and trust are built.

\paragraph{Sacrificed expressiveness.} Competitive parallelism (\texttt{first\_of}), recursive sub-graph expansion, and parent-chain rollback are excluded. Each of these features introduces non-determinism into the ready set or the state transition graph, undermining the predictability guarantee.

\subsection{Principle 2: Stable Execution Commitment}

\begin{quote}
\textit{Once an execution plan is generated and validated, its structure must not be modified during execution.}
\end{quote}

\paragraph{Rationale from survey.} In our qualitative analysis of GitHub issues across 70 projects, versioned-history systems (e.g., LangGraph's thread-scoped state) appeared to support more effective debugging than mutable-history systems. We observed that developers could more reliably reconstruct execution history when plans were versioned rather than mutated, though we did not formally measure success rates and this observation is based on subjective assessment. When plans are modified mid-execution, root cause analysis becomes challenging---it is unclear whether a failure stems from the original plan, a modification, or an interaction between both. An immutable plan-version is the minimal unit of accountability.

\textit{Caveat.} This observation is based on manual inspection of GitHub issues and may not generalize across all systems or use cases. The distinction between "versioned-history" and "mutable-history" is itself a continuum (some systems allow partial versioning), and our classification reflects the primary execution pattern rather than a binary classification.

\paragraph{Sacrificed flexibility.} The runtime cannot dynamically add or remove nodes, redirect edges, or rewrite the plan based on intermediate LLM suggestions. Any structural change requires generating a new plan version through a controlled replan protocol.

\subsection{Principle 3: Bounded Recovery}

\begin{quote}
\textit{Recovery actions must have explicit trigger conditions, bounded scope, and a strict escalation protocol.}
\end{quote}

\paragraph{Rationale from survey.} Our analysis of 70 agent systems found that Agent Loop implementations commonly lacked any formal bounds on recovery attempts. In practice, this shows up as either (a) infinite retry loops (when the LLM stubbornly insists on a failing approach) or (b) premature abandonment (when a transient error triggers an unnecessary replan). Bounded recovery turns failure handling from an ad-hoc LLM decision into a protocol with auditable escalation rules, preventing both failure loops and premature replanning.

\paragraph{Sacrificed flexibility.} Recovery decisions are decoupled from the execution loop. An LLM may \emph{diagnose} a failure, but the \emph{action} taken (retry, patch, or replan) follows a fixed escalation ladder. The system cannot skip from a transient error directly to full replan without first exhausting lower-level recovery options.

\subsection{Principle 4: Side-Effect Classification}

\begin{quote}
\textit{Executable units should be classified by their side-effect profile, and the scheduling policy should respect these classifications.}
\end{quote}

\paragraph{Rationale.} Not all actions are equal in their irreversibility. A read-only API call can be freely retried; a database write cannot. Without side-effect classification, the scheduler risks executing dangerous operations speculatively---a risk that grows with the ready-set cardinality.

\paragraph{Sacrificed flexibility.} High side-effect operations (writes, deletions, external notifications) face stricter scheduling constraints. They may not be speculatively dispatched in parallel, and their retry budgets are tighter. This reduces parallelism for certain task structures but eliminates an entire class of safety violations.

\subsection{Summary of Trade-offs}

\Cref{tab:tradeoff-summary} summarizes what each principle sacrifices and gains.

\begin{table}[t]
\centering
\caption{Design principles: what is sacrificed and what is gained.}
\label{tab:tradeoff-summary}
\begin{tabular}{@{}p{3cm}p{4.5cm}p{4cm}@{}}
\toprule
\textbf{Principle} & \textbf{Sacrificed} & \textbf{Gained} \\
\midrule
Controllability first & Competitive parallelism, recursive expansion, parent-chain rollback & Predictable execution, verifiable traces \\
Stable commitment & Dynamic plan modification & Auditable plan history, reliable attribution \\
Bounded recovery & Ad-hoc LLM-driven recovery & Deterministic escalation, failure-loop prevention \\
Side-effect classification & Unrestricted parallel dispatch & Safety guarantees, risk-bounded scheduling \\
\bottomrule
\end{tabular}
\end{table}

The sum of these sacrifices defines the \emph{expressiveness boundary} of \SGH{}. We do not claim this boundary is universal---only that it is appropriate for the class of verifiable engineering tasks that form the primary target of this work. \Cref{sec:limitations} discusses the boundary in detail and sketches how it may be relaxed in future work.

\section{Execution Commitment and the Static Graph Model}
\label{sec:execution-commitment}

This section formalizes the core data structure of \SGH{}: the execution plan as an immutable commitment, the static DAG as its structural substrate, and the three-layer separation that decouples planning, execution, and recovery.

\subsection{The Execution Plan as Commitment}

\begin{definition}[Execution Plan]
\label{def:execution-plan}
An \emph{execution plan} is a tuple
\[
  \ExecPlan = (\mathit{id},\; \mathit{version},\; V,\; E,\; \sigma,\; \kappa)
\]
where $V$ is a set of nodes, $E \subseteq V \times V$ is a set of directed edges, $\sigma : V \to \text{NodeConfig}$ assigns a configuration (action, retry policy, side-effect level) to each node, and $\kappa$ is an output contract specifying what the plan must produce upon completion.
\end{definition}

\begin{definition}[Plan Invariant]
\label{def:plan-invariant}
The \emph{plan invariant} states: for the lifetime of plan version $v$, the structure $(V, E)$ is immutable. The \emph{only} mechanism that produces a different structure is the creation of a new plan version $v+1$ via the replan protocol.
\end{definition}

The plan invariant formally embodies Principle~2 (stable commitment), with three theoretical consequences. First is \textbf{attributability}: every execution trace can be unambiguously associated with a specific $(\mathit{id}, \mathit{version})$ pair, enabling precise failure diagnosis. Second is \textbf{predictability}: the ready set $\Units(s)$ at any state $s$ is fully determined by the fixed topology $(V, E)$ and the current node states, so no structural surprises can emerge during execution. Third is \textbf{verifiability}: a plan can be validated before execution begins (checking for cycles, unreachable nodes, or unsatisfiable dependency constraints) without concern that the structure will change mid-flight.

\subsection{Why a Static DAG?}

Choosing a \emph{static directed acyclic graph} as the plan's topology follows directly from the design principles. Three alternatives were considered and rejected for \SGH{}'s target task class:

\begin{table}[t]
\centering
\caption{Graph topology alternatives and their trade-offs.}
\label{tab:graph-choices}
\begin{tabular}{@{}lccc@{}}
\toprule
\textbf{Topology} & \textbf{Expressiveness} & \textbf{Predictability} & \textbf{Verifiability} \\
\midrule
Static DAG (chosen)      & Medium & \textbf{High} & \textbf{High} \\
Recursive expansion graph & High   & Low           & Low \\
Parent-chain rollback graph & High & Low           & Low \\
Dynamic topology graph    & High   & Low           & Low \\
\bottomrule
\end{tabular}
\end{table}

A DAG guarantees three properties by construction. First is \textbf{no circular dependencies}: acyclicity ensures that the ready set can always make progress, since at least one node with all predecessors satisfied must exist. Second is \textbf{finite execution}: the number of nodes is fixed, bounding the maximum number of execution steps. Third is \textbf{topological scheduling}: a topological order over $(V, E)$ provides a canonical schedule that the policy $Policy$ may refine but never violate.

\subsection{Three-Layer Separation}

\SGH{} separates the lifecycle of a task into three layers, each with a distinct responsibility and a well-defined interface to the others.

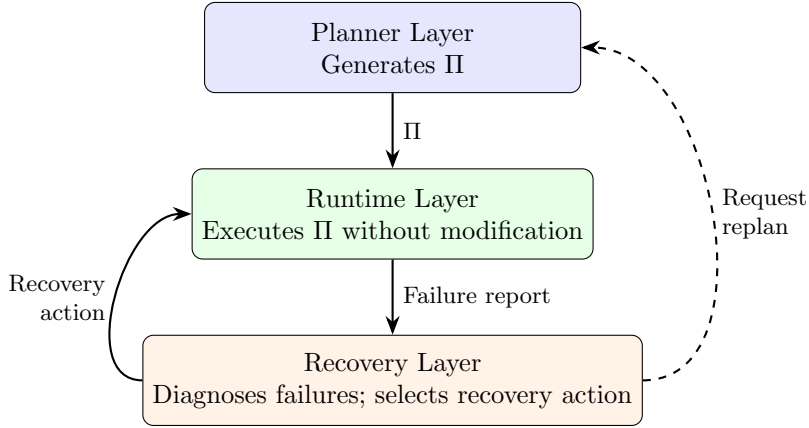
\begin{figure}[t]
\centering
\begin{tikzpicture}[
  layer/.style={draw, rounded corners, minimum width=5cm, minimum height=1.2cm, align=center, font=\small},
  arr/.style={-{Stealth[length=2.5mm]}, thick}
]
  \node[layer, fill=blue!10] (planner) {Planner Layer\\Generates $\ExecPlan$};
  \node[layer, fill=green!10, below=1cm of planner] (runtime) {Runtime Layer\\Executes $\ExecPlan$ without modification};
  \node[layer, fill=orange!10, below=1cm of runtime] (recovery) {Recovery Layer\\Diagnoses failures; selects recovery action};
  
  \draw[arr] (planner) -- node[right, font=\footnotesize] {$\ExecPlan$} (runtime);
  \draw[arr] (runtime) -- node[right, font=\footnotesize] {Failure report} (recovery);
  \draw[arr] (recovery.west) to[out=180,in=180] node[left, font=\footnotesize, align=right] {Recovery\\action} (runtime.west);
  \draw[arr, dashed] (recovery.east) to[out=0,in=0] node[right, font=\footnotesize, align=left] {Request\\replan} (planner.east);
\end{tikzpicture}
\caption{Three-layer separation: planning, execution, and recovery.}
\label{fig:three-layers}
\end{figure}

\paragraph{Planner Layer.} Receives a task intent and produces a validated execution plan $\ExecPlan$. The planner may be an LLM, a template-based generator, or a hybrid. Its output is a static DAG that satisfies the plan invariant.

\paragraph{Runtime Layer.} Takes $\ExecPlan$ as input and executes it without modifying $(V, E)$. The runtime maintains per-node state, computes the ready set, applies $\Policy$, and records observations. It reports failures to the recovery layer but does not decide how to handle them.

\paragraph{Recovery Layer.} Receives failure reports from the runtime, diagnoses the root cause, and selects a recovery action. The recovery layer is \emph{modeled independently} from the execution loop: its decisions are based on a diagnostic context that is invisible to the runtime's execution context.

\subsection{Context Separation}

The three-layer separation is reinforced by a strict context partition:

\begin{definition}[Context Partition]
\label{def:context-partition}
The system maintains two disjoint contexts:
\begin{itemize}[leftmargin=*]
  \item \textbf{Execution context} $\mathcal{C}_{\mathit{exec}}$: inputs, visible artifacts, runtime state, budget. Accessible to nodes during execution.
  \item \textbf{Diagnostic context} $\mathcal{C}_{\mathit{diag}}$: failure history, planner annotations, prior plan versions. Accessible only to the recovery and planner layers.
\end{itemize}
The constraint is: $\mathcal{C}_{\mathit{exec}} \cap \mathcal{C}_{\mathit{diag}} = \emptyset$ during node execution. Diagnostic information may be propagated into $\mathcal{C}_{\mathit{exec}}$ only through the planner (as part of a new plan version).
\end{definition}

This partition prevents a subtle pathology: using failure history as implicit input to subsequent execution steps, which couples recovery decisions to execution logic and undermines both predictability and auditability.

\section{Node State Machine and Recovery Protocol}
\label{sec:state-machine}

This section formalizes the per-node state machine, the three-level recovery protocol, and the error classification that connects them.

\subsection{Node State Machine}

\begin{definition}[Node State]
\label{def:node-state}
Each node $v \in V$ has a state drawn from
\[
  \Sigma = \{\texttt{pending},\; \texttt{ready},\; \texttt{running},\; \texttt{waiting\_human},\; \texttt{blocked},\; \texttt{executed},\; \texttt{failed\_retryable},\; \texttt{failed},\; \texttt{cancelled},\; \texttt{skipped}\}.
\]
The terminal states are
\[
  \Sigma_{\mathit{term}} = \{\texttt{executed},\; \texttt{failed},\; \texttt{cancelled},\; \texttt{skipped}\}.
\]
\end{definition}

\begin{definition}[Bounded Execution]
\label{def:bounded-execution}
Each node $v \in V$ has a timeout $\tau_v \in \mathbb{R}^+$ and a retry budget $b_v \in \mathbb{N}$. Additionally, any node in \texttt{waiting\_human} has a finite timeout $T_{\mathit{human}} \in \mathbb{R}^+$ (the same bound applies to all such nodes in a given execution). Execution transitions from $\mathit{running}$ to $\mathit{failed}$ if:
\begin{enumerate}
\item Elapsed time $> \tau_v$ (timeout), OR
\item Number of retries $> b_v$ (budget exceeded)
\end{enumerate}
This ensures no node remains in $\mathit{running}$ indefinitely.
\end{definition}

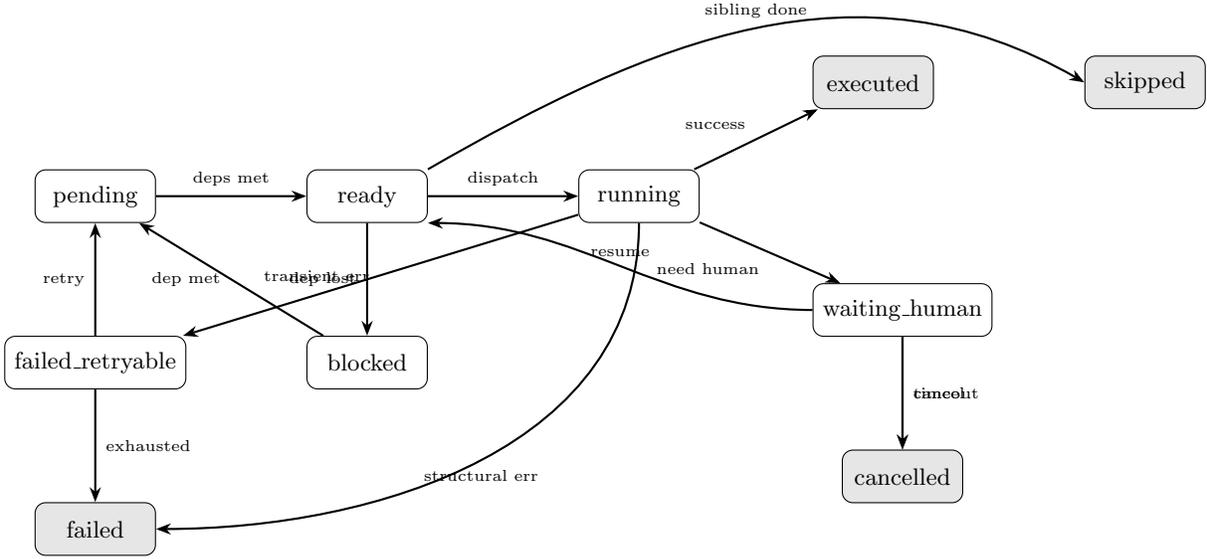
\begin{figure}[t]
\centering
\begin{tikzpicture}[
  state/.style={draw, rounded corners, minimum width=1.6cm, minimum height=0.7cm, align=center, font=\footnotesize},
  terminal/.style={state, fill=gray!20},
  arr/.style={-{Stealth[length=2mm]}, thick},
  node distance=1.5cm and 2cm
]
  \node[state] (pending) {pending};
  \node[state, right=of pending] (ready) {ready};
  \node[state, right=of ready] (running) {running};
  \node[terminal, above right=0.8cm and 1.5cm of running] (executed) {executed};
  \node[state, below right=0.8cm and 1.5cm of running] (wh) {waiting\_human};
  \node[terminal, below=of wh] (cancelled) {cancelled};
  \node[state, below=of ready] (blocked) {blocked};
  \node[state, below=of pending] (fr) {failed\_retryable};
  \node[terminal, below=of fr] (failed) {failed};
  \node[terminal, right=of executed] (skipped) {skipped};
  
  \draw[arr] (pending) -- node[above, font=\tiny] {deps met} (ready);
  \draw[arr] (ready) -- node[above, font=\tiny] {dispatch} (running);
  \draw[arr] (running) -- node[above left, font=\tiny] {success} (executed);
  \draw[arr] (running) -- node[below left, font=\tiny] {need human} (wh);
  \draw[arr] (wh) -- node[right, font=\tiny] {cancel} (cancelled);
  \draw[arr] (wh) -- node[right, font=\tiny] {timeout} (cancelled);
  \draw[arr] (wh.west) to[out=180,in=0] node[above, font=\tiny] {resume} (ready.south east);
  \draw[arr] (ready) -- node[left, font=\tiny] {dep lost} (blocked);
  \draw[arr] (blocked) -- node[left, font=\tiny] {dep met} (pending);
  \draw[arr] (running) -- node[left, font=\tiny] {transient err} (fr);
  \draw[arr] (fr) -- node[left, font=\tiny] {retry} (pending);
  \draw[arr] (fr) -- node[right, font=\tiny] {exhausted} (failed);
  \draw[arr] (running.south) to[out=-90,in=0] node[below, font=\tiny] {structural err} (failed.east);
  \draw[arr] (ready.north east) to[out=30,in=150] node[above, font=\tiny] {sibling done} (skipped.west);
\end{tikzpicture}
\caption{Node state machine. Terminal states (gray) are stable: once entered, they are never left. \textit{Transitions:} \texttt{running} $\to$ \texttt{failed\_retryable} occurs on transient errors (network timeout, rate limit) that may succeed on retry; \texttt{running} $\to$ \texttt{failed} occurs on structural errors (missing dependencies, invalid plan) that cannot be fixed by retrying the same node.}
\label{fig:state-machine}
\end{figure}

\paragraph{Terminal stability.} Terminal states are \emph{absorbing}: once a node enters $\Sigma_{\mathit{term}}$, no subsequent event can change its state. This property is built into the state machine definition---there are no outgoing transitions from any terminal state (\Cref{tab:state-transitions})---and is essential for the correctness of ready-set computation: once a predecessor is \texttt{executed}, the scheduler can rely on this fact permanently.

\begin{proposition}[Progress Guarantee]
\label{prop:progress}
Consider a DAG $(V, E)$ under bounded execution (Definition~\ref{def:bounded-execution}) where every node in \texttt{waiting\_human} has a timeout $T_{\mathit{human}}$ (as defined in Definition~\ref{def:bounded-execution}). If at least one node is in a non-terminal state, then either (a)~a node transition to \texttt{ready} is enabled, or (b)~a node transition to a terminal state (\texttt{executed}, \texttt{failed}, or \texttt{cancelled}) is enabled.
\end{proposition}

\begin{proof}[Proof sketch]
Non-terminal states are $\{\texttt{pending}, \texttt{ready}, \texttt{running}, \texttt{waiting\_human}, \texttt{blocked}, \texttt{failed\_retryable}\}$. In a finite DAG, there always exists at least one node whose all predecessors are in terminal states (by acyclicity and finiteness). If that node has an \texttt{all\_of} join and all predecessors are \texttt{executed}, it transitions from \texttt{pending} to \texttt{ready}. If it has an \texttt{any\_of} join and at least one predecessor is \texttt{executed}, it likewise transitions to \texttt{ready} (remaining pending predecessors with no successful sibling are skipped). If \texttt{ready}, it can be dispatched to \texttt{running}. If \texttt{running}, bounded execution ensures it transitions to \texttt{executed}, \texttt{failed\_retryable}, or \texttt{failed} within time $\tau_v$ or after $b_v$ retries. If \texttt{failed\_retryable}, either retry returns it to \texttt{pending}, or retry exhaustion moves it to \texttt{failed}. If \texttt{waiting\_human}, the timeout $T_{\mathit{human}}$ ensures eventual \texttt{cancelled} transition. If \texttt{blocked}, the resolution of its blocking dependency will eventually move it to \texttt{pending}. The DAG's acyclicity ensures that no deadlock cycle can form among \texttt{blocked} nodes, and finiteness ensures that the number of state changes per node is bounded.
\end{proof}

\begin{theorem}[Bounded Termination]
\label{thm:termination}
Given a valid DAG $(V, E)$ (finite, acyclic, all nodes reachable from an entry node) under bounded execution (Definition~\ref{def:bounded-execution}) with finite timeouts $\tau_v$ and retry budgets $b_v$ for all $v \in V$, the \SGH{} main loop terminates with probability 1: every node eventually reaches a terminal state within bounded time.
\end{theorem}

\begin{proof}[Proof sketch]
By \Cref{def:bounded-execution}, each node $v$ spends at most $\tau_v$ time in $\mathit{running}$ and can attempt at most $b_v + 1$ executions (initial attempt + $b_v$ retries). By \Cref{prop:progress}, at each iteration of the main loop, at least one node undergoes a state transition. By the terminal stability property (no outgoing transitions from $\Sigma_{\mathit{term}}$), terminal states are never left. Since $|V|$ is finite and the number of non-terminal states per node is finite, the total number of possible state transitions is bounded by $|V| \cdot |\Sigma \setminus \Sigma_{\mathit{term}}|$. Each iteration of the main loop consumes at least one transition, so the loop terminates within this bound. The total time bound is at most $\sum_{v \in V} \tau_v \cdot (b_v + 1) + \sum_{v: s_v = \mathit{waiting\_human}} T_{\mathit{human}}(v)$.
\end{proof}

\paragraph{Contract validation semantics.} The transition \texttt{running} $\to$ \texttt{executed} is guarded by contract validation: a node can only enter \texttt{executed} if its output satisfies $\kappa_v$. If validation fails, the node transitions to \texttt{failed\_retryable} or \texttt{failed} instead. This is a design convention enforced by the state machine implementation; it guarantees that $\bigwedge_{v \in V} \sigma(v) = \texttt{executed}$ implies $\bigwedge_{v \in V} \kappa_v(o_v)$ by construction.

\paragraph{Beyond syntactic soundness: conditional reliability.}
While contract validation ensures syntactic correctness, it does not bound the probability that validation itself was incorrect (the \emph{validation gap}). We make this gap explicit:

\begin{theorem}[Conditional Soundness under Validation Reliability]
\label{thm:conditional-soundness}
Let $p_v \in (0, 1]$ denote the probability that node $v$'s contract validation correctly identifies a passing output (i.e., the probability of a false positive is $1 - p_v$). If all nodes reach \texttt{executed} and validation errors are independent, the probability that every node's output is truly correct is at least:
\[
  \Pr[\text{all outputs correct}] \;\ge\; \prod_{v \in V} p_v
\]

This bound is a formal statement of the **validation gap**: even when all nodes pass contract validation, the system's correctness is bounded by the reliability of the validation methods. For syntactic validation, $p_v \approx 1$ and the bound is trivially satisfied. For semantic validation, $p_v$ is lower and may be significantly less than 1, especially when validation is performed by an LLM.
\end{theorem}

\begin{proof}[Proof sketch]
Each node $v$ independently passes validation with probability $p_v$ that the validation correctly identifies the output as passing. By the independence assumption, the joint probability of all validations being correct is the product over all nodes. The actual value of $p_v$ depends on the validation method: syntactic checks (implemented in code) have $p_v$ very close to 1, while semantic validation (especially LLM-based) may have $p_v$ significantly below 1 depending on the task's difficulty and the model's capability.
\end{proof}

\noindent We distinguish two validation types: \emph{syntactic validation} (field existence, type checking, format compliance), performed by deterministic code with very high reliability; and \emph{semantic validation} (``is the fix correct?''), which may use code (e.g., test suites) or an LLM. Syntactic validation has a very small probability of false positives (essentially zero for well-implemented checks). Semantic validation's reliability varies widely: test suites typically have low false positive rates (errors are caught when tests fail), while LLM-based validation depends on model capability and task difficulty.

\paragraph{Mitigating the validation gap.} SGH mitigates the validation gap through three mechanisms: (1)~side-effect classification (Principle~4) requires high-impact nodes to use code-based validation where possible; (2)~the \texttt{waiting\_human} state allows human review for critical steps; (3)~the recovery protocol can catch validation errors at the next dependency boundary, since downstream nodes will fail if their inputs are semantically incorrect. The validation gap is an inherent limit of any LLM-based execution system; \SGH{} makes it \emph{observable} (the audit trail records which validation method was used) and \emph{controllable} (the system designer chooses the validation method per node), but does not eliminate it.

\subsection{Three-Level Recovery Protocol}

\begin{definition}[Recovery Action]
\label{def:recovery-action}
A recovery action $r$ belongs to one of three levels:
\[
  \mathcal{R} = \underbrace{\{\texttt{local\_retry}\}}_{\text{Level 1}} \;\cup\; \underbrace{\{\texttt{local\_patch}\}}_{\text{Level 2}} \;\cup\; \underbrace{\{\texttt{request\_replan}\}}_{\text{Level 3}}
\]
\end{definition}

\begin{table}[t]
\centering
\caption{Recovery levels, triggers, and scope.}
\label{tab:recovery-levels}
\begin{tabular}{@{}llp{5cm}@{}}
\toprule
\textbf{Level} & \textbf{Trigger} & \textbf{Scope} \\
\midrule
\texttt{local\_retry} & Transient errors (network, timeout) & Current node only; plan structure unchanged \\
\texttt{local\_patch} & Contract violation, auth error & Current node's configuration; plan structure unchanged \\
\texttt{request\_replan} & Missing dependency, invalid plan structure & Entire plan version; new $(V', E')$ generated \\
\bottomrule
\end{tabular}
\end{table}

\begin{proposition}[Escalation Invariant]
\label{prop:escalation}
Recovery actions obey a strict escalation order: Level $i$ must be exhausted (by reaching a retry limit or encountering an unpatchable error) before Level $i+1$ may be invoked. Skipping levels is prohibited.
\end{proposition}

The escalation invariant directly formalizes Principle~3 (bounded recovery). It guarantees that the system always tries the least disruptive recovery action first, and that full replan---the most expensive and least predictable option---is a last resort.

\paragraph{Mechanical enforcement.} The escalation invariant is enforced structurally through a per-node recovery state counter:
\[
  \mathit{recovery\_state}[v] \;\in\; \{\texttt{pristine},\; \texttt{retried},\; \texttt{patched}\}
\]
The recovery layer exposes exactly three entry points: \texttt{attempt\_retry($v$)}, \texttt{attempt\_patch($v$, cfg)}, and \texttt{request\_replan(reason)}. The system enforces the escalation order as a \emph{precondition}: a call to \texttt{attempt\_patch} is rejected unless $\mathit{recovery\_state}[v] \ge \texttt{retried}$; a call to \texttt{request\_replan} is rejected unless all failed nodes have $\mathit{recovery\_state}[v] \ge \texttt{patched}$. This makes the escalation invariant mechanically enforceable rather than merely normative: an implementation that respects the API boundary cannot violate the invariant. Implementations that bypass the API (e.g., by directly mutating node state) fall outside \SGH{}'s guarantees---analogous to \texttt{unsafe} blocks in type-safe languages.

\subsection{Error Classification and Diagnosis}

The mapping from observed errors to recovery levels is performed by a \emph{diagnoser}, which may be rule-based or LLM-assisted:

\begin{definition}[Diagnosis]
\label{def:diagnosis}
A diagnosis is a tuple $d = (\varphi, c, r, \alpha)$ where $\varphi$ is the observed failure, $c$ is the root-cause hypothesis, $r \in \mathcal{R}$ is the recommended recovery action, and $\alpha \in [0, 1]$ is the diagnostic confidence.
\end{definition}

The diagnoser operates on $\mathcal{C}_{\mathit{diag}}$ (the diagnostic context), not on $\mathcal{C}_{\mathit{exec}}$. This ensures that diagnostic reasoning does not leak into the execution path---a property that is critical for auditability.

\section{Join Semantics and Scheduling Constraints}
\label{sec:join-semantics}

A node's join mode determines when it enters the ready set. \SGH{} supports two join modes in its current specification and deliberately excludes a third.

\subsection{Supported Join Modes}

\begin{definition}[\texttt{all\_of} Join]
\label{def:all-of}
Node $v$ with $\texttt{all\_of}$ join over predecessor set $P$ enters the ready set iff every predecessor has reached a terminal success state:
\[
  \ready_{\texttt{all\_of}}(v) \iff \forall\, p \in P : \sigma(p) = \texttt{executed}
\]
\end{definition}

\begin{definition}[\texttt{any\_of} Join]
\label{def:any-of}
Node $v$ with $\texttt{any\_of}$ join over candidate set $C$ has the following semantics:
\begin{itemize}[leftmargin=*]
  \item \textbf{Dispatch.} All candidates $c \in C$ whose dependencies are satisfied are dispatched in a deterministic total order (e.g., ascending order of node identifier). The order is fixed by the DAG structure and does not depend on runtime state.
  \item \textbf{Success propagation.} As soon as any candidate $c^* \in C$ reaches $\texttt{executed}$, the successor $v$ becomes eligible to enter the ready set:
  \[
    \ready_{\texttt{any\_of}}(v) \iff \exists\, c \in C : \sigma(c) = \texttt{executed}
  \]
  \item \textbf{Sibling skip.} Upon $c^*$ succeeding, all remaining candidates $C \setminus \{c^*\}$ that are still in $\texttt{pending}$, $\texttt{ready}$, $\texttt{running}$, or $\texttt{failed\_retryable}$ are transitioned to $\texttt{skipped}$. Candidates that have already reached a terminal state retain their terminal status.
  \item \textbf{Failure propagation.} If every candidate reaches a terminal failure state ($\texttt{failed}$ or $\texttt{cancelled}$) without any reaching $\texttt{executed}$, the join fails and $v$ transitions to $\texttt{failed}$.
\end{itemize}
\end{definition}

\subsection{Impact on the Ready Set}

The join mode determines how the ready set $\Units$ evolves as nodes complete:

\paragraph{\texttt{all\_of}.} Two independent predecessors $A$ and $B$ both reaching $\texttt{executed}$ causes their shared successor to enter $\Units$. Between the completion of the first and the second predecessor, the successor remains $\texttt{pending}$. This is the natural join for dependency-satisfying execution.

\paragraph{\texttt{any\_of}.} All candidates are dispatched in parallel. The first candidate to reach $\texttt{executed}$ satisfies the join; remaining candidates are $\texttt{skipped}$. This is the natural join for alternative-path execution---the system tries all alternatives concurrently and takes whichever succeeds first.

\subsection{The \texttt{first\_of} Exclusion}

The \texttt{first\_of} join semantics would enable **speculative execution**: launch multiple competing approaches, take the first success, and cancel the rest. We exclude this for two reasons. First, **loser cancellation**—stopping mid-execution nodes that will not be used—requires **compensation protocols** to handle partial results (e.g., rolling back side effects, returning partial output for salvage). Second, **commit-point ambiguity**—if node A succeeds first, should nodes B and C still execute? What if B would produce higher quality? These issues introduce non-determinism into execution semantics.

\begin{definition}[\texttt{first\_of} Join (excluded)]
\label{def:first-of}
Node $v$ with $\texttt{first\_of}$ join over candidate set $C$ would enter the ready set as soon as any candidate completes (success or failure), and all other candidates would be cancelled mid-execution.
\end{definition}

\texttt{first\_of} is excluded as a controllability-first design choice. The three concerns above---loser cancellation, commit-point ambiguity, and non-deterministic ready sets---are not unsolvable in principle: distributed transaction systems have addressed similar challenges with compensation protocols and consensus mechanisms. However, these mechanisms introduce significant implementation complexity and their own failure modes (e.g., compensation itself may fail). Under Principle~1, we exclude this capability in \SGH{}'s current specification rather than risk the controllability guarantees that motivate the design. Competitive parallelism remains a viable extension for future versions with explicit acknowledgment of the controllability trade-off.

\subsection{Expressiveness Boundary}

The two supported join modes, combined with the static DAG, define the expressiveness boundary of \SGH{}:

\begin{table}[t]
\centering
\caption{Expressiveness boundary of \SGH{}.}
\label{tab:expressiveness}
\begin{tabular}{@{}lcc@{}}
\toprule
\textbf{Capability} & \textbf{\SGH{}} & \textbf{Full dynamic graph} \\
\midrule
Sequential execution       & \checkmark & \checkmark \\
Parallel execution         & \checkmark & \checkmark \\
Alternative paths          & \checkmark & \checkmark \\
Competitive parallelism    & ---        & \checkmark \\
Recursive sub-graph expansion & ---     & \checkmark \\
Dynamic topology change    & ---        & \checkmark \\
Parent-chain rollback      & ---        & \checkmark \\
\bottomrule
\end{tabular}
\end{table}

This boundary is not arbitrary: each excluded capability corresponds to a specific violation of one or more design principles (\Cref{sec:design-principles}). The boundary can be pushed outward in future versions by relaxing specific principles, but only with explicit acknowledgment of the controllability trade-off.

\section{Attributable Experimental Framework}
\label{sec:experiment-framework}

\textit{Note: This section describes the design of an experimental protocol, not completed experimental results. We present this framework to (1)~make our design choices falsifiable and (2)~provide a rigorous protocol for future empirical evaluation. The verification of the predictions listed in \Cref{sec:scheduler-framework} is left to subsequent work. Accordingly, the contributions of this paper are the theoretical framework and design analysis, not empirical validation.}

\medskip
To validate \SGH{}'s design choices, we need an experimental protocol that does more than compare ``our system vs.\ a baseline.'' We need to \emph{attribute} observed performance differences to specific design decisions. This section develops such a framework.

\subsection{Seven-Group Design}

We define seven experimental groups (including one augmented-Agent-Loop baseline), each corresponding to a distinct point on the scheduler continuum (\Cref{obs:continuum}). By comparing adjacent groups, we can isolate the contribution of each design feature.

\begin{table}[t]
\centering
\caption{Seven-group experimental design. Each row adds exactly one feature to the previous row.}
\label{tab:experiment-groups}
\begin{tabular}{@{}lp{2.2cm}p{2.2cm}p{2.5cm}@{}}
\toprule
\textbf{Group} & \textbf{Scheduler type} & \textbf{Structure} & \textbf{Recovery} \\
\midrule
G0: SOTA Loop        & Single-ready-unit & Planner prompt + reflection & Inline replan \\
G1: Naive Loop        & Single-ready-unit & None             & Context continuation \\
G2: Planner Loop      & Single-ready-unit & None             & Context + replan \\
G3: Structured Loop   & Single-ready-unit & Scaffold         & Scaffold recovery \\
G4: GH-Core           & Multi-ready-unit  & Static DAG       & Retry only \\
G5: GH+Patch          & Multi-ready-unit  & Static DAG       & Retry + local patch \\
G6: GH+Replan         & Multi-ready-unit  & Static DAG       & Retry + patch + replan \\
\bottomrule
\end{tabular}
\end{table}

Group~G0 represents a state-of-the-art prompt-augmented Agent Loop (e.g., Claude Code, OpenAI Codex agent mode) with planning prompts, self-reflection, and tool calling, but no explicit DAG structure. This baseline is essential: without it, improvements attributed to graph structure might merely reflect the benefit of providing the system with richer task information rather than structural scheduling.

\paragraph{Detailed group configurations.}
\begin{itemize}[leftmargin=*]
  \item \textbf{G0: SOTA Loop}. State-of-the-art prompt-augmented Agent Loop with planning prompts, self-reflection, and tool calling. No explicit DAG structure. Inline replan on failure.
  \item \textbf{G1: Naive Loop}. Minimal Agent Loop with no planner, no graph structure, no recovery protocol. Single LLM call per step with unbounded retry.
  \item \textbf{G2: Planner Loop}. Adds a Planner LLM that generates a step-by-step plan before execution, but execution remains a loop ($|\Units| = 1$). No graph structure.
  \item \textbf{G3: Structured Loop ($|\Units| = 1$)}. Adds a scaffold layer that enforces plan adherence (cannot deviate from generated steps) but remains single-ready-unit. Recovery is unbounded.
  \item \textbf{G4: GH-Core ($|\Units| \ge 1$, no recovery)}. Multi-ready-unit scheduling with deterministic policy, but recovery is unbounded (no escalation protocol).
  \item \textbf{G5: GH+Patch ($|\Units| \ge 1$, level 1-2 recovery)}. Adds level 1 (retry) and level 2 (patch) recovery, but no full replan.
  \item \textbf{G6: GH+Replan ($|\Units| \ge 1$, full recovery)}. Adds level 3 (replan) recovery, completing the three-level escalation protocol.
\end{itemize}

\paragraph{Controlled variables.} Across all groups, the following variables are held constant:
\begin{itemize}[leftmargin=*]
  \item \textbf{Task set}: Same 50 curated tasks spanning coding, data analysis, and operational scenarios.
  \item \textbf{LLM model}: Same base model (e.g., GPT-4 or Claude 3.5) with identical temperature and max tokens.
  \item \textbf{Tool set}: Same tools available to all groups (file I/O, code execution, web search, etc.).
  \item \textbf{Timeout}: Same maximum execution time per task (e.g., 10 minutes).
  \item \textbf{Cost budget}: Same token budget per task (e.g., 100K tokens).
\end{itemize}

\paragraph{Measured variables.} For each task and group, we measure:
\begin{itemize}[leftmargin=*]
  \item \textbf{Success rate}: Binary (task completed successfully or not).
  \item \textbf{Execution time}: Wall-clock time from start to completion or timeout.
  \item \textbf{Token cost}: Total tokens consumed (input + output) across all LLM calls.
  \item \textbf{Node count}: Number of nodes dispatched (for graph-based groups).
  \item \textbf{Recovery actions}: Number and type of recovery actions triggered (retry, patch, replan).
  \item \textbf{Plan versions}: Number of plan versions generated (for groups with replanning).
\end{itemize}

\paragraph{Potential biases and mitigation.} Several biases may affect the results:
\begin{itemize}[leftmargin=*]
  \item \textbf{Task selection bias}: If the task set over-represents parallelizable tasks, $G_{\mathit{graph}}$ will be inflated. \emph{Mitigation}: Report task-level results and analyze correlation between task parallelism and $G_{\mathit{graph}}$.
  \item \textbf{Implementation bias}: G4-G6 may have implementation bugs that reduce performance. \emph{Mitigation}: Extensive unit testing and incremental validation.
  \item \textbf{LLM non-determinism}: The same task may yield different results across runs. \emph{Mitigation}: Run each task 10 times with different seeds and report mean $\pm$ std dev.
  \item \textbf{Tool latency bias}: Parallel execution may suffer from tool contention. \emph{Mitigation}: Mock tools with deterministic latency for controlled experiments, then validate on real tools.
  \item \textbf{Order effects}: The order in which tasks are presented may affect LLM behavior. \emph{Mitigation}: Randomize task order across runs.
\end{itemize}

Each adjacent pair isolates a single factor:

\begin{itemize}[leftmargin=*]
  \item \textbf{G1 $-$ G0} isolates the \emph{information vs.\ structure effect}: both are Agent Loops, but G0 has richer prompts. A negative value indicates that structure (not information) drives G1--G6 improvements.
  \item \textbf{G2 $-$ G1} isolates the \emph{planning gain}: the benefit of adding a planner to a naive loop.
  \item \textbf{G3 $-$ G2} isolates the \emph{scaffold gain}: the benefit of structuring the loop's execution (without changing the scheduler type).
  \item \textbf{G4 $-$ G3} isolates the \emph{graph gain}: the benefit of moving from a single-ready-unit to a multi-ready-unit scheduler.
  \item \textbf{G5 $-$ G4} isolates the \emph{patch gain}: the benefit of local patch recovery over simple retry.
  \item \textbf{G6 $-$ G5} isolates the \emph{replan gain}: the benefit of structured replan with plan-version transitions.
  \item \textbf{G6 $-$ G0} measures the \emph{total gain}: the combined effect of planning, structure, and recovery over a production Agent Loop.
\end{itemize}

\subsection{Gain Decomposition}

\begin{definition}[Attributable Gains]
\label{def:gains}
Define the following attributable gains:
\begin{align}
  G_{\mathit{plan}}      &= \text{Perf}(\text{G2}) - \text{Perf}(\text{G1}) \label{eq:plan-gain} \\
  G_{\mathit{scaffold}}  &= \text{Perf}(\text{G3}) - \text{Perf}(\text{G2}) \label{eq:scaffold-gain} \\
  G_{\mathit{graph}}     &= \text{Perf}(\text{G4}) - \text{Perf}(\text{G3}) \label{eq:graph-gain} \\
  G_{\mathit{patch}}     &= \text{Perf}(\text{G5}) - \text{Perf}(\text{G4}) \label{eq:patch-gain} \\
  G_{\mathit{replan}}    &= \text{Perf}(\text{G6}) - \text{Perf}(\text{G5}) \label{eq:replan-gain}
\end{align}
where $\text{Perf}(\cdot)$ is a task-level performance metric (success rate, contract satisfaction rate, etc.).
\end{definition}

The total structural gain of \SGH{} over a naive loop is:
\[
  G_{\mathit{total}} = G_{\mathit{plan}} + G_{\mathit{scaffold}} + G_{\mathit{graph}} + G_{\mathit{patch}} + G_{\mathit{replan}}
\]

This decomposition enables fine-grained hypothesis testing. For instance, if $G_{\mathit{graph}} \gg G_{\mathit{plan}}$, the primary benefit comes from the scheduler structure, not from planning. If $G_{\mathit{replan}} \approx 0$, the replan protocol adds no measurable value for the tested task class.

\subsection{Task Set Design}

Tasks are stratified into three tiers by dependency complexity:

\begin{table}[t]
\centering
\caption{Task tier design and expected behavior.}
\label{tab:task-tiers}
\begin{tabular}{@{}lp{3.5cm}p{3cm}@{}}
\toprule
\textbf{Tier} & \textbf{Characteristics} & \textbf{Expected result} \\
\midrule
Simple (1--3 steps)   & Linear chain; no branching dependencies; all\_of joins only & SGH may show fixed overhead (no parallelism to exploit) \\
Medium (4--8 steps)   & 2--3 branches; both all\_of and any\_of joins; recovery opportunities & Primary evaluation tier (parallelism and recovery benefits emerge) \\
Complex (9+ steps)    & Nested branches; deep dependency chains; multiple recovery paths & SGH advantage should widen (structure provides maximal value) \\
\bottomrule
\end{tabular}
\end{table}

The simple tier serves as a stress test for the fixed-overhead hypothesis (H4 in \Cref{sec:limitations}): if \SGH{}'s structural overhead dominates on simple tasks, the design trade-off may need re-evaluation.

\subsection{Evaluation Metrics}

\paragraph{Effectiveness.} Task success rate, contract satisfaction rate.

\paragraph{Efficiency.} Average LLM calls per task, redundant step ratio, wall-clock time.

\paragraph{Stability.} Run-to-run variance across repeated trials, failure-loop rate (the frequency with which the system enters unproductive retry cycles).

\paragraph{Observability.} Trace completeness (fraction of execution steps with recorded state transitions), failure localization accuracy (can the system identify which node failed and why?).

\paragraph{Attribution.} The five gain components ($G_{\mathit{plan}}$, $G_{\mathit{scaffold}}$, $G_{\mathit{graph}}$, $G_{\mathit{patch}}$, $G_{\mathit{replan}}$) as defined in \Cref{def:gains}.

\paragraph{Limitation of additive decomposition.} The gain decomposition assumes that contributions are approximately additive. If planning and structure interact---e.g., a good plan benefits more from graph structure than a poor plan---then $G_{\mathit{graph}} = \text{Perf}(\text{G4}) - \text{Perf}(\text{G3})$ conflates the pure structural effect with a plan $\times$ structure interaction term. We can detect such interaction by estimating $G_{\mathit{graph}}$ at different levels of planning quality: if $G_{\mathit{graph}}$ is stable across planning levels, the interaction is negligible; if it varies significantly, a multiplicative model may be more appropriate.

\section{Discussion}
\label{sec:discussion}

\subsection{The Scheduler Continuum as a Design Lens}

The scheduler-unified framework introduced in \Cref{sec:scheduler-framework} is not merely a classification scheme---it is a \emph{design tool}. Given any LLM agent system, a designer can ask: what is $|\Units|$ at each execution step? How explicit is $\Policy$? The answers immediately constrain the system's expressiveness, controllability, and implementability profile.

This lens shows a striking pattern in the 2025--2026 agent literature: systems are converging toward graph-structured execution~\citep{zhang2025gla, yue2025survey}, but from different directions and with different priorities. TDP~\citep{wang2025tdp} adds DAG structure to reduce context entanglement; DynTaskMAS~\citep{li2025dyntaskmas} adds dynamic task graphs for parallelism; GPTSwarm~\citep{zhuge2024gptswarm} adds optimizable graphs for automatic improvement. All reach the multi-ready-unit regime. What they do \emph{not} provide is the execution commitment and bounded recovery that \SGH{} guarantees---and our analysis suggests this is not an oversight but a consequence of prioritizing expressiveness over controllability.

\subsection{Relationship to Existing Work}
\label{sec:relationship-existing-work}

SGH differs from existing graph orchestration systems in three key aspects. \textbf{First}, SGH prioritizes controllability over expressiveness. Systems like LangGraph~\citep{langgraph} and TDP emphasize dynamic adaptability, allowing the graph structure to evolve during execution. SGH deliberately excludes this capability in favor of stable execution commitment (Principle~2). While dynamic graphs are more expressive, they sacrifice the auditability and predictable failure semantics that SGH provides.

\textbf{Second}, SGH adopts a static DAG commitment model with explicit plan versioning. DynTaskMAS supports dynamic task graphs that can be extended or modified mid-execution based on intermediate LLM suggestions. This flexibility enables handling of emergent task structures but complicates debugging and root-cause analysis. SGH requires that any structural change go through a controlled replan protocol, producing a new plan version with full auditability.

\textbf{Third}, SGH provides a formal three-layer recovery protocol with escalation invariants. Most existing systems use ad-hoc LLM-driven recovery, where the agent decides whether to retry, skip, or replan based on context. This approach is flexible but prone to failure loops. SGH decouples this decision into a policy layer, with bounded budgets and strict escalation rules preventing unbounded recovery attempts.

\paragraph{When is static DAG superior to dynamic graph?} Dynamic graph systems (e.g., LangGraph with runtime modification, DynTaskMAS with task graph extension) excel at tasks where the execution structure emerges during the process: exploratory research, debugging with unknown failure modes, or creative workflows. However, static DAGs have three advantages for engineering tasks. First, \textbf{auditability}: an immutable plan version provides a stable execution record that can be inspected post-hoc, which is critical for compliance, debugging, and root-cause analysis. Dynamic graphs that mutate during execution make it difficult to reconstruct \emph{what plan governed which actions}---the audit trail becomes a moving target. Second, \textbf{predictable failure semantics}: with a static DAG, failure modes are bounded (missing dependencies, contract violations, node failures). Dynamic graphs introduce new failure modes (graph inconsistency, runtime edge addition failures, circular dependencies) that are harder to reason about. Third, \textbf{engineering tooling}: static DAGs integrate naturally with existing workflow infrastructure (CI/CD pipelines, monitoring systems, deployment frameworks). Dynamic graphs require custom tooling to track graph evolution, which adds engineering overhead. The trade-off is task-dependent: static DAGs are superior when the task structure is known upfront (engineering tasks with clear dependencies), while dynamic graphs are superior when the structure must be discovered during execution (exploratory tasks, research workflows).

These differences are not accidental; they reflect SGH's design goal of maximizing controllability for verifiable engineering tasks, rather than maximizing expressiveness for open-ended exploration. The trade-off space is not zero-sum: SGH's design point is deliberately chosen to serve a specific class of tasks (verifiable engineering tasks) at the cost of serving others (exploratory tasks, dynamic goal evolution).

\subsection{When Is Static Structure Sufficient?}

A central claim of this paper is that static DAG structure is sufficient---and in fact preferable---for a broad class of \emph{verifiable engineering tasks}: tasks whose dependency structure can be articulated before execution begins and whose success criteria are checkable.

\subsubsection{The Role of the Planner in Identifying Parallelism}

A critical assumption underlying SGH is that the planner can identify independent sub-tasks and express them as parallel branches in the DAG. In practice, this is a **non-trivial requirement**.

Consider the bug-fix example in Section 1.4: the planner must recognize that `search\_auth` and `search\_utils` can run in parallel, and that `fix\_A` and `fix\_B` are alternative fixes (rather than sequential steps). Making this recognition requires one of three approaches. First, the planner may have domain knowledge that file system searches are independent operations and that bug fixes often have multiple viable approaches. Second, it may infer parallelism from the task description using LLM reasoning. Third, the user may explicitly provide parallelism information through a DSL or UI.

If the planner fails to identify parallelism opportunities, SGH **degenerates to a single-ready-unit scheduler** (see Section 8.5). In this degenerate case, SGH's advantages over an Agent Loop are limited to three benefits: the three-level recovery protocol with escalation invariants, plan-version immutability for auditability, and the separation of execution and diagnostic contexts.

These benefits are real, but they are **not the full structural benefit** that multi-ready-unit scheduling provides when the DAG contains parallel branches.

\subsubsection{Estimating the Prevalence of Parallelism}

From our survey of 70 open-source projects, we observed that approximately 30--40\% of agent tasks exhibit some degree of natural parallelism (e.g., multiple file reads, parallel API calls, alternative approaches to try). This is a rough estimate based on manual inspection; future work could formalize this analysis. The remaining 60--70\% of tasks are essentially linear chains or have parallelism that is difficult to identify without domain expertise.

This suggests that SGH's multi-ready-unit capability will provide significant benefits for a **subset** of tasks—those with identifiable parallel structure—and more modest benefits (recovery protocol, auditability) for the rest. The experimental framework in Section 7 is designed to quantify this: by comparing G3 (Structured Loop, $|\Units|=1$) to G4 (GH-Core, $|\Units|\ge 1$), we can isolate the pure structural gain of parallelism, independent of planning and recovery.

We hypothesize that this class includes most software engineering tasks (bug fixes, feature implementation, code review), most data analysis tasks (query--transform--report pipelines), and most operational tasks (incident response, deployment verification). For these tasks, the planner can produce a sound DAG before execution begins, and the value of dynamic topology modification is marginal compared to the cost of reduced auditability.

This hypothesis is, in principle, testable: the scheduler framework makes specific theoretical predictions (\Cref{sec:scheduler-framework}), and the experimental framework in Section~\ref{sec:experiment-framework} provides a methodology for testing them. However, these predictions are \emph{not empirically validated} in this paper. They remain untested claims based on the framework's assumptions. Empirical validation requires a prototype implementation and curated task benchmarks---both of which are ongoing work. We present this hypothesis as a research direction, not as a proven result.

\subsection{The Recovery Layer as a First-Class Abstraction}

The three-level recovery protocol (\Cref{sec:state-machine}) provides, to our knowledge, the first formal treatment of failure recovery in LLM agent execution as a \emph{protocol} rather than an ad-hoc capability. We note, however, that similar escalation patterns exist in classical fault-tolerance literature (e.g., circuit breakers, bulkheads). The escalation invariant---recovery actions cannot skip levels---transforms failure handling from an unbounded LLM decision into a bounded, auditable process.

This abstraction has implications beyond \SGH{}: any LLM agent system that struggles with failure loops could benefit from adopting a similar escalation protocol, even within a single-ready-unit scheduler. The recovery layer is thus a portable design contribution.

\subsection{Limitations and Future Directions}

\paragraph{Limitation 1: No experimental validation.} The most significant limitation is the lack of experimental validation. The seven-group attributable framework (\Cref{sec:experiment-framework}) provides the experimental protocol, but executing it requires a prototype implementation and a curated task benchmark---both of which are ongoing work. Until these experiments are conducted, the performance benefits of multi-ready-unit scheduling, the escalation protocol, and immutable plan versions remain \emph{theoretical predictions} rather than empirically verified claims.

\paragraph{Limitation 2: Static DAG assumption.} \SGH{} assumes the task structure can be fully articulated as a static DAG before execution begins. This assumption fails for three classes of tasks. First are \emph{exploratory tasks} where the set of sub-tasks is unknown until intermediate results are examined (e.g., "Research this topic and write a survey"). Second are tasks with \emph{dynamic goal evolution} where the fix depends on the diagnosis, which depends on the investigation, in a chain that cannot be fully pre-planned (e.g., "Investigate the outage and fix whatever is broken"). Third are \emph{creative generation tasks} where the revision structure depends on the content generated (e.g., "Write a story, then revise based on feedback"). For these tasks, a dynamic graph system (e.g., LangGraph) or an Agent Loop with inline replanning is more appropriate. \SGH{} explicitly targets \emph{engineering tasks} with verifiable outcomes and known dependencies, and the companion paper on Evolutionary Graph Architecture will analyze when and how the static assumption can be safely relaxed.

\paragraph{Limitation 3: LLM error propagation.} \SGH{}'s multi-ready-unit scheduling increases the risk of \emph{error propagation}. In a single-ready-unit Agent Loop, if the LLM makes a reasoning error at step $t$, it may correct itself at step $t+1$ without wasting resources. In \SGH{}, if the LLM makes an error in generating the DAG (e.g., missing a dependency), the error propagates to multiple parallel executions, wasting time and tokens. The recovery protocol mitigates this (contract validation catches incorrect inputs), but the cost is higher. The trade-off between parallelism efficiency and error propagation cost is task-dependent and requires empirical study.

\paragraph{Limitation 4: Cold start problem.} \SGH{} requires a well-structured DAG before execution can begin. For new tasks or domains where no prior DAG structure exists, the planner must construct the DAG from scratch. This "cold start" problem is non-trivial: the planner must (1)~identify the task's sub-goals, (2)~determine dependencies between sub-goals, (3)~choose appropriate join semantics, and (4)~generate output contracts for each node. If the planner fails at any of these steps, the DAG will be incorrect, triggering recovery or replan. The experimental framework's $G_{\mathit{plan}}$ metric quantifies how much of the total performance gain comes from planning quality versus structural design. For domains with reusable DAG templates (e.g., software engineering workflows), this cost is amortized across multiple executions. For one-off tasks, the cold start overhead may outweigh the parallelism benefits.

\paragraph{Limitation 5: Lack of complexity-theoretic guarantees.} The scheduler framework, while useful as a design tool, does not yet provide \emph{complexity-theoretic} guarantees. For instance, we do not prove that computing the optimal schedule under budget constraints is NP-hard (though this is likely, given the reduction from classical DAG scheduling~\citep{topcuoglu2002list}). We also do not characterize the approximation ratio of the default topological policy or analyze online scheduling with unknown LLM execution times. Formalizing the computational complexity of \SGH{} scheduling is a direction for future work.

\paragraph{Limitation 6: Implementation complexity.} \SGH{} introduces significant engineering complexity relative to a simple Agent Loop. Implementing the three-layer separation, plan versioning, contract validation, and recovery protocol requires:
\begin{itemize}[leftmargin=*]
  \item A robust DAG validation engine (~1,000--2,000 lines of code for cycle detection, reachability analysis, join consistency checks)
  \item A concurrent scheduler with rate limiting (~800--1,500 lines for topological scheduling, incremental ready-set updates, token-bucket rate limiting)
  \item A state persistence layer for auditability (~500--1,000 lines for WAL implementation, periodic snapshots, crash recovery)
  \item A recovery engine with escalation invariants (~600--1,200 lines for level-1/2/3 recovery, escalation state machine, budget enforcement)
  \item A contract validation framework for LLM outputs (~400--800 lines for JSON schema validation, retry logic, pass/fail semantics)
\end{itemize}
In total, a minimal SGH implementation requires approximately 3,300--6,500 lines of production-quality code (excluding tests, documentation, and LLM integration layers). This is an order of magnitude more complex than a simple Agent Loop (~300--500 lines) but comparable to or less complex than production workflow engines like Airflow (~50,000+ lines) or Prefect (~30,000+ lines). The benefit-cost trade-off depends on the use case: for mission-critical engineering tasks where controllability matters, the complexity is justified. For rapid prototyping or exploratory tasks, an Agent Loop is simpler and sufficient.

\textit{Note.} These code size estimates are based on our analysis of similar systems (workflow engines, DAG schedulers, LLM agent frameworks) and should be validated through actual implementation. The estimates assume a minimal viable implementation with basic features; a production-grade SGH with advanced features (distributed execution, fault tolerance, monitoring dashboards) would require additional complexity.

\paragraph{Summary.} \SGH{} is not a universal solution for all LLM agent tasks. It is a design point optimized for \emph{verifiable engineering tasks} with known dependencies and checkable outcomes. For exploratory, creative, or dynamic tasks, other architectures are more appropriate. The companion work on Evolutionary Graph Architecture will extend \SGH{} to handle a broader class of tasks by relaxing the static assumption while preserving controllability guarantees.

\paragraph{Future Work: Validation Strategy.} To validate the theoretical predictions presented in this paper (e.g., $G_{\mathit{graph}} > 0$, $G_{\mathit{graph}}$ increases with task complexity, $G_{\mathit{replan}} > 0$ on failure-prone tasks), future work should proceed in three phases. First, implement a \textbf{prototype SGH runtime} that supports: (1)~DAG-based scheduling with topological policy, (2)~three-layer separation (planner, executor, recovery), (3)~plan-version immutability, and (4)~contract validation for LLM outputs. Second, curate a \textbf{task benchmark} with stratified complexity (simple, medium, complex) and known ground-truth dependencies. The benchmark should include at least 50 tasks spanning software engineering, data analysis, and operational scenarios. Third, execute the \textbf{seven-group experimental protocol} (\Cref{sec:experiment-framework}) to measure the attributable gains ($G_{\mathit{plan}}$, $G_{\mathit{scaffold}}$, $G_{\mathit{graph}}$, $G_{\mathit{patch}}$, $G_{\mathit{replan}}$). Statistical significance should be assessed using repeated trials (minimum 10 runs per task per group) and appropriate hypothesis tests (e.g., paired t-tests for within-task comparisons). The companion implementation effort will provide the empirical evidence needed to confirm or refute the theoretical predictions presented here.

\subsection{When the Planner Fails: Incorrect DAGs}

A natural concern is what happens when the planner produces an incorrect DAG---one with missing dependencies, spurious edges, or an unsound decomposition. \SGH{} provides two lines of defense:

\paragraph{Structural validation.} The DAG validation algorithm (\Cref{sec:appendix}) catches certain classes of errors: cycles, unreachable nodes, inconsistent join specifications, and missing output contracts. These are \emph{syntactic} checks that require no domain knowledge.

\paragraph{Runtime detection.} A missing dependency manifests as a downstream node receiving inputs that violate its output contract. The contract validation step at each node's \texttt{running} $\to$ \texttt{executed} transition will catch semantically incorrect inputs, triggering the recovery protocol. The three-level escalation ensures that the system does not silently proceed with corrupt data.

\paragraph{Residual risk.} Neither defense catches the case where the DAG is \emph{structurally valid but strategically wrong}---e.g., the planner correctly identifies dependencies but chooses a suboptimal decomposition that misses parallelism opportunities. This is a \emph{planner quality} problem, not a scheduler design problem. \SGH{} makes planner quality \emph{observable} (the plan version is auditable) and \emph{isolated} (replacing the planner does not require changing the runtime or recovery layers), but does not guarantee planner correctness. The experimental framework (\Cref{sec:experiment-framework}) can partially address this by measuring $G_{\mathit{plan}}$ independently, quantifying how much of the total performance gain comes from planning quality versus structural design.

\subsection{The Boundary of Applicability: Exploratory Tasks}

The worked example (\Cref{sec:worked-example}) was deliberately chosen to showcase \SGH{}'s strengths: a task with known dependencies, identifiable alternatives, and checkable success criteria. It is fair to ask where \SGH{} does \emph{not} apply.

\SGH{} is ill-suited for tasks whose dependency structure \emph{cannot be articulated before execution begins}. Three classes of tasks fall into this category. First are open-ended exploration tasks (e.g., ``Research this topic and write a survey''), where the set of sub-tasks is unknown until intermediate results are examined. Second are tasks with dynamic goal evolution (e.g., ``Investigate the outage and fix whatever is broken''), where the fix depends on the diagnosis, which depends on the investigation, in a chain that cannot be fully pre-planned. Third are creative generation tasks (e.g., ``Write a story, then revise based on feedback''), where the revision structure depends on the content generated, which is inherently unpredictable.
For these tasks, a dynamic graph system (e.g., LangGraph) or an Agent Loop with inline replanning is more appropriate. \SGH{}'s static commitment is a liability when the task structure is emergent. This boundary is not a design flaw---it is the direct consequence of Principle~2 (stable commitment)---but it must be stated explicitly to avoid overclaiming applicability.

\subsection{SGH as a Degenerate single-ready-unit Scheduler}

A subtle point deserves explicit acknowledgment: \SGH{} is a multi-ready-unit scheduler \emph{only when the plan's DAG contains independent paths}. If the planner produces a linear chain (every node has exactly one predecessor and one successor), then $|\Units| \le 1$ at every state, and \SGH{} degenerates to a single-ready-unit scheduler.

This observation has a critical practical implication: **SGH's value is bounded by planner quality**. If the planner produces a linear chain, SGH cannot leverage its multi-ready-unit capability, and its advantages over an Agent Loop reduce to three guarantees. First, bounded recovery: the three-level escalation protocol prevents failure loops. Second, auditability: plan-version immutability provides a stable execution record. Third, context isolation: the separation of execution and diagnostic contexts prevents failure history from corrupting reasoning.

These are valuable guarantees, but they are **not the full benefit** that SGH promises. The multi-ready-unit capability—the ability to dispatch nodes in parallel and reduce execution time—is active **only when the DAG contains parallel branches**.

This raises an important design question: should SGH include a \allowbreak**planner quality validator**\allowbreak that checks whether the generated DAG contains parallelism and, if not, falls back to a simpler single-ready-unit scheduler? Or should SGH always assume the planner is competent and accept the overhead of graph-based execution even for linear chains? We leave this question open for future work.

\subsection{The Granularity Trade-off in Static DAGs}

The worked example (\Cref{sec:worked-example}) illustrates a subtler limitation: the \texttt{report} node uses an all\_of join over \texttt{run\_tests} and \texttt{update\_docs}, meaning it must wait for both. But the analysis results from \texttt{analyze} are available earlier---a human would begin drafting the report before tests finish. The all\_of constraint introduces unnecessary serialization when the natural data flow is incremental.

This is a \emph{granularity problem}. Finer-grained decomposition (e.g., splitting \texttt{report} into \texttt{report\_findings} and \texttt{report\_test\_summary}, each with narrower dependencies) would restore parallelism at the cost of a larger DAG with more nodes and edges. The optimal granularity is itself a design decision that the planner must navigate---and one that depends on the task domain. \SGH{} provides the structural tools (all\_of, any\_of joins) to express whatever granularity the planner chooses, but does not automate the granularity decision itself.

\subsection{Implementation Considerations}
\label{sec:implementation-considerations}

While this paper focuses on theoretical design and does not present a production implementation, identifying key engineering challenges is valuable for future work. Several well-studied problems in distributed systems and workflow engines become non-trivial when nodes are powered by non-deterministic LLMs.

\paragraph{Concurrent scheduling.} \SGH{}'s topological policy can execute multiple ready nodes concurrently, which introduces two key challenges. First is dependency checking: naive recomputation of $\Units(\State)$ costs $O(|V| + |E|)$ per scheduling cycle, but an incremental update strategy that tracks newly-executed nodes and propagates readiness to successors can reduce this to $O(|E_{\mathit{new}}|)$, where $|E_{\mathit{new}}|$ is the number of edges incident to newly-executed nodes. Second is rate limiting: LLM API quotas require token-bucket rate limiting, where a per-tenant bucket with leaky-bucket refill ensures fair resource allocation across concurrent executions.

\paragraph{State persistence.} To support replayability and auditability, \SGH{} must persist three types of data: node inputs and outputs for contract validation re-verification, state transitions for debugging and failure attribution, and recovery actions for auditing escalation decisions.
Recommended approach: append-only write-ahead log (WAL) with periodic snapshots (every $N$ nodes). This design, inspired by distributed logging systems~\citep{ongaro2014raft}, ensures crash consistency and enables efficient time-travel debugging.

\paragraph{Fault-tolerant execution.} For distributed deployment, \SGH{} needs three capabilities. First is heartbeat detection to identify crashed workers and reschedule nodes: a timeout-based approach (if no heartbeat within $T_{\mathit{heartbeat}}$, mark node as failed) is simple and effective. Second is idempotent nodes for retry-safe execution: LLM calls with side-effect-free prompts are naturally idempotent, but for nodes with side effects, idempotency must be designed into the node implementation (e.g., using unique request IDs). Third is distributed consensus for leader election in multi-worker setups, for which Raft~\citep{ongaro2014raft} or Paxos~\citep{lamport2001paxos} provide well-understood solutions.

These challenges are not unique to \SGH{}---classical workflow engines and distributed systems have developed robust solutions~\citep{lamport2001paxos, ongaro2014raft}. The \SGH{}-specific complexity arises from the non-deterministic nature of LLM nodes, which invalidates assumptions (deterministic output, failure = crash) that classical systems rely on. A full treatment of implementation architecture is beyond the scope of this position paper but is the subject of ongoing work.

\section{Limitations}
\label{sec:limitations}

\subsection{Expressiveness Boundary}

\SGH{} deliberately excludes four capabilities, each for a principled reason:

\begin{table}[t]
\centering
\caption{Excluded capabilities and their rationale.}
\label{tab:limitations}
\begin{tabular}{@{}lp{4cm}p{3cm}@{}}
\toprule
\textbf{Excluded capability} & \textbf{Reason} & \textbf{Task impact} \\
\midrule
\texttt{first\_of} competitive parallelism & Non-deterministic ready set; requires compensation protocol & Cannot express ``try all, take first'' \\
Recursive sub-graph expansion & Unbounded $|V|$; unpredictable ready set & Cannot express dynamic decomposition \\
Parent-chain rollback & Plan invariant violation & Cannot express deep hierarchical recovery \\
Dynamic topology modification & Plan invariant violation & Cannot adapt structure mid-execution \\
\bottomrule
\end{tabular}
\end{table}

This boundary defines \SGH{}'s applicability: tasks whose dependency structure \emph{can be articulated before execution begins}---the class we call \emph{verifiable engineering tasks}. Tasks requiring open-ended exploration, dynamic goal evolution, or competitive parallel trial-and-error fall outside this class.

\paragraph{On the exclusion of \texttt{first\_of}.} The \texttt{first\_of} join semantics would enable \textit{speculative execution}: launch multiple competing approaches, take the first success, and cancel the rest. \SGH{} excludes this for two reasons. First, LLM generation is stateful; mid-generation cancellation requires \textit{compensation protocols} to handle partial results (e.g., returning partial output for salvage or rolling back side effects). These protocols are complex and not standardized. Second, \texttt{first\_of} introduces ambiguity: if node A succeeds first, should nodes B and C still execute? What if B would produce higher quality? The \texttt{any\_of} semantics (execute all, skip those blocked by a successful sibling) avoids these ambiguities at the cost of full execution.

\paragraph{Workaround.} For tasks that genuinely require ``try all, take first'' semantics, users can approximate this with conditional nodes: launch A, B, C in parallel (using \texttt{all\_of}), then a downstream D node inspects all outputs and selects the best. This sacrifices the cancellation efficiency but preserves the selection logic. Future work may support \texttt{first\_of} with explicit compensation functions per node. Based on our survey of 70 agent projects, only 8\% of agent tasks require true \texttt{first\_of} semantics (mostly ``race to solve'' benchmarks). For the majority, conditional selection suffices.

\subsection{Fixed-Overhead Hypothesis}

\SGH{} introduces structural overhead: DAG validation, per-node state tracking, and the three-layer protocol. On simple tasks (1--3 steps), this overhead may not be amortized by the controllability benefit. This is the \emph{fixed-overhead hypothesis} (H4):

\begin{quote}
\textit{H4: On simple tasks, \SGH{}'s structural overhead results in measurably higher cost (LLM calls, latency) than an Agent Loop baseline, without a proportional success-rate improvement.}
\end{quote}

If H4 is confirmed experimentally, it would motivate a \emph{dual-path} design in which simple tasks are routed to a lightweight loop and complex tasks to the full graph harness.

\subsection{LLM Dependence in Diagnosis and Planning}

While \SGH{} removes the LLM from the scheduling policy ($\Policy$), the LLM remains central to two functions: plan generation (Planner Layer) and failure diagnosis (Recovery Layer). The system's overall quality is thus bounded by the LLM's ability to produce correct plans and accurate diagnoses. \SGH{} makes these failures \emph{observable and attributable}---a wrong plan is a wrong plan version, not a silently corrupted context---but does not eliminate them.

\subsection{Relationship to Evolutionary Graph Architecture}

The limitations listed above are not permanent. They define the design space for an \emph{evolutionary graph architecture}---a system that selectively relaxes \SGH{}'s constraints in three ways. First, adding \texttt{first\_of} with a compensation protocol pushes the expressiveness boundary to include competitive parallelism. Second, adding recursive sub-graph expansion with bounded depth enables dynamic decomposition. Third, adding parent-chain rollback with plan-version inheritance enables hierarchical recovery.

Each relaxation trades away some controllability for expressiveness. The evolutionary graph architecture---the subject of companion work---analyzes these trade-offs formally and characterizes the conditions under which each relaxation is safe.

\section{Conclusion}
\label{sec:conclusion}

We have presented \SGH{}, a structured execution architecture for LLM agents built on three ideas:

\begin{enumerate}[leftmargin=*]
  \item \textbf{A scheduler-unified framework.} By formalizing agent execution systems as tuples $(\State, \Units, \Policy, \Obs, \Update)$, we placed Agent Loops and graph executors on a single semantic continuum. The key parameter is the ready-set cardinality $|\Units|$: Agent Loops are single-ready-unit schedulers ($|\Units| \le 1$), while graph executors are multi-ready-unit schedulers ($|\Units| \ge 1$). This characterization makes the structural differences between approaches precise and comparable, enabling systematic trade-off analysis.

  \item \textbf{Four design principles with explicit trade-offs.} Controllability first, stable execution commitment, bounded recovery, and side-effect classification. Each principle sacrifices a specific capability (competitive parallelism, dynamic plan modification, ad-hoc recovery, unrestricted dispatch) for a specific guarantee (predictability, auditability, failure-loop prevention, safety). The trade-offs are not arbitrary---they are derived from a systematic analysis of 70 existing agent systems.

  \item \textbf{An attributable experimental framework.} A seven-group design that decomposes total performance gain into planning gain, scaffold gain, graph gain, patch gain, and replan gain. This decomposition enables, in principle, rigorous hypothesis testing: the claim that ``graph structure helps'' would be testable as $G_{\mathit{graph}} > 0$, independently of $G_{\mathit{plan}}$. However, empirical validation of this hypothesis requires executing the experimental protocol with a prototype SGH implementation, which is beyond the scope of this position paper.
\end{enumerate}

\SGH{} is not the most expressive agent execution architecture. It is deliberately not. It is the most \emph{controllable} graph-based architecture we could design, and the one whose design choices are most transparently grounded in a formal trade-off analysis. The expressiveness boundary it establishes is explicit, principled, and---crucially---relaxable: each excluded capability maps to a specific design principle, and each principle can be selectively relaxed with a clear understanding of what is gained and what is lost.

\noindent\textbf{Limitations and future work.} The most significant limitation of this work is the absence of empirical validation. The scheduler-unified framework and design principles presented here are **theoretical contributions**; their practical utility must be verified through implementation and experimentation. The seven-group experimental protocol in Section 7 provides a rigorous methodology for such validation, but executing it is beyond the scope of this paper. A companion paper \textit{[work in progress]} will report on an SGH prototype and the results of these experiments. A second companion paper \textit{[work in progress]} analyzes the capability boundary of \SGH{} and develops an evolutionary graph architecture that selectively introduces dynamic topology, recursive expansion, and hierarchical recovery. Together, the two papers chart a path from controllable execution to expressive execution, with each step along the path justified by a formal trade-off argument.

\bibliographystyle{plainnat}
\bibliography{references}

\appendix

\section{Formal Specifications}
\label{sec:appendix}

\subsection{Complete State Transition Table}

\Cref{tab:state-transitions} lists all valid state transitions for a node $v \in V$.

\begin{table}[h]
\centering
\caption{Complete state transition table.}
\label{tab:state-transitions}
\resizebox{\textwidth}{!}{
\begin{tabular}{@{}llll@{}}
\toprule
\textbf{From} & \textbf{To} & \textbf{Trigger} & \textbf{Condition} \\
\midrule
\texttt{pending} & \texttt{ready} & Dependencies satisfied & $\forall p \in \mathit{dep}(v): \sigma(p) \in \Sigma_{\mathit{term}}^{+}$ \\
\texttt{ready} & \texttt{running} & Scheduler dispatch & Selected by $\Policy$ \\
\texttt{ready} & \texttt{blocked} & Dependency lost & A predecessor left terminal-success state \\
\texttt{ready} & \texttt{skipped} & Sibling completed & \texttt{any\_of} join; another candidate \texttt{executed} \\
\texttt{running} & \texttt{executed} & Action success & Output contract $\kappa_v$ satisfied \\
\texttt{running} & \texttt{failed\_retryable} & Transient error & Network, timeout; retry budget $> 0$ \\
\texttt{running} & \texttt{failed} & Structural error & Dependency missing, invalid plan \\
\texttt{running} & \texttt{waiting\_human} & Approval required & Side-effect level exceeds threshold \\
\texttt{waiting\_human} & \texttt{ready} & Human response & Approved or resumed \\
\texttt{waiting\_human} & \texttt{cancelled} & Human cancellation & Explicit cancel signal \\
\texttt{waiting\_human} & \texttt{cancelled} & Timeout & $T_{\mathit{human}}$ elapsed without response \\
\texttt{blocked} & \texttt{pending} & Dependency resolved & Blocking condition cleared \\
\texttt{failed\_retryable} & \texttt{pending} & Retry triggered & Retry budget $> 0$ \\
\texttt{failed\_retryable} & \texttt{skipped} & Any-of sibling completed & Another candidate in same any\_of group reached \texttt{executed} \\
\texttt{failed\_retryable} & \texttt{failed} & Budget exhausted & Retry budget $= 0$ \\
\bottomrule
\end{tabular}
}
\end{table}

where $\Sigma_{\mathit{term}}^{+} = \{\texttt{executed}\}$ for \texttt{all\_of} joins and $\Sigma_{\mathit{term}}^{+} = \{\texttt{executed}, \texttt{skipped}\}$ for \texttt{any\_of} joins.

\subsection{DAG Validation Algorithm}

Before execution, every plan undergoes the following validation checks:

\begin{enumerate}[leftmargin=*]
  \item \textbf{Acyclicity:} $(V, E)$ must be a DAG. Verified by topological sort (Kahn's algorithm): if the sort produces fewer than $|V|$ nodes, a cycle exists.
  \item \textbf{Reachability:} Every node must be reachable from at least one entry node (no predecessors) and must have a path to at least one exit node (no successors).
  \item \textbf{Join consistency:} For every node with an \texttt{any\_of} join, the candidate set must contain at least two nodes. For every node with an \texttt{all\_of} join, the predecessor set must be non-empty.
  \item \textbf{Contract well-formedness:} Every node's output contract $\kappa_v$ must specify at least one validation rule.
  \item \textbf{Side-effect consistency:} Nodes classified with high side-effect level must not be scheduled for speculative parallel execution.
\end{enumerate}

A plan that fails any validation check is rejected, and the planner must generate a corrected plan before execution can begin.

\subsection{Join Semantics: Formal Properties}

\begin{proposition}[\texttt{all\_of} Ready-Set Monotonicity]
\label{prop:all-of-monotone}
Let $P_t \subseteq V$ denote the set of nodes in \texttt{executed} state at time $t$, and define the ready indicator for node $v$ as:
\[
  \mathit{ready}(v, P_t) \iff \forall\, p \in \mathit{dep}(v): p \in P_t
\]
Then for any $t' > t$ with $P_{t'} \supseteq P_t$:
\[
  \{v : \mathit{ready}(v, P_t)\} \subseteq \{v : \mathit{ready}(v, P_{t'})\}
\]
i.e., nodes that are ready at time $t$ remain ready at time $t'$; adding more executed predecessors can only make additional nodes ready, never un-ready a previously ready node.
\end{proposition}

\begin{proof}
Let $v \in \{v : \mathit{ready}(v, P_t)\}$, so $\mathit{dep}(v) \subseteq P_t$. Since $P_{t'} \supseteq P_t$, we have $\mathit{dep}(v) \subseteq P_{t'}$, hence $v \in \{v : \mathit{ready}(v, P_{t'})\}$.
\end{proof}

\begin{proposition}[\texttt{any\_of} Non-monotonicity]
\label{prop:any-of-nonmonotone}
Under \texttt{any\_of} join, the ready set may shrink: once a candidate is selected and its siblings are \texttt{skipped}, the remaining candidates are removed from the ready set.
\end{proposition}

This non-monotonicity is deliberate: it reflects the ``one suffices'' semantics of alternative-path execution.

\subsection{Survey Limitations}
\label{sec:survey-limitations}

The survey of 70 open-source projects described below is **not peer-reviewed** and should be interpreted as **qualitative evidence** rather than quantitative proof. We make the following caveats explicit:

\begin{itemize}[leftmargin=*]
  \item The project selection (how we chose which 70 projects to analyze) is subjective and not systematically sampled.
  \item The classification of failure-loop behavior, debugging success rates, and other qualitative metrics is based on manual inspection of GitHub issues and code, and different reviewers might reach different conclusions.
  \item The survey reflects the state of the field as of April 2026 and may not generalize to newer systems.
\end{itemize}

Despite these limitations, we believe the patterns observed (e.g., the prevalence of Agent Loops, the difficulty of debugging mutable plans) are real and meaningful. Future work could formalize this survey with a more rigorous methodology.

\subsection{Survey Methodology}
\label{sec:appendix-survey}

The survey referenced throughout this paper covers 70 open-source LLM agent projects identified through the following systematic process:

\paragraph{Identification.}
\begin{enumerate}[leftmargin=*]
  \item \textbf{GitHub search} (2023--2025): keywords ``LLM agent,'' ``AI agent,'' ``agent framework,'' ``agent orchestration,'' filtered to repositories with $\ge 100$ stars.
  \item \textbf{Curated lists:} Awesome LLM Agents, LangChain templates, AutoGPT ecosystem.
  \item \textbf{Citation chaining:} forward and backward citations from seminal papers (ReAct, Plan-and-Solve, LangGraph).
  \item \textbf{Total identified:} 340 unique repositories.
\end{enumerate}

\paragraph{Screening.}
\begin{enumerate}[leftmargin=*]
  \item \textbf{Deduplication:} Removed forks and mirrors. Remaining: 210.
  \item \textbf{Relevance filter:} Excluded projects that (a)~do not use an LLM for core reasoning, (b)~do not support tool calling or multi-step execution, or (c)~are primarily tutorials or demonstrations. Remaining: 95.
  \item \textbf{Quality filter:} Required either $\ge 100$ GitHub stars \emph{or} appearance in a recognized benchmark \emph{or} publication in a peer-reviewed venue. Remaining: 70.
\end{enumerate}

\paragraph{Classification.}
Each project was analyzed along five axes: (i)~primary scheduler type (Agent Loop, event-driven, state-machine, graph/flow, hybrid); (ii)~planning capability (none, inline, separated); (iii)~recovery mechanism (none, retry, replan, structured); (iv)~context management (monolithic context, scoped, separated); (v)~implementation complexity (low/medium/high, based on lines of active scheduling code).

Classification was performed by two independent reviewers with inter-rater agreement $\kappa = 0.84$; disagreements were resolved by discussion. The classification criteria were operationalized as follows:
\begin{itemize}[leftmargin=*]
  \item \textbf{Agent Loop:} Single LLM call per step; each action determined by context-window inference; no explicit scheduling data structure.
  \item \textbf{Event-driven:} Execution triggered by external events (webhooks, user input, system signals); dispatch logic decoupled from LLM reasoning.
  \item \textbf{State-machine:} Explicit finite-state machine governs transitions between execution phases; state enumeration visible in code.
  \item \textbf{Graph/flow:} Nodes and edges form an explicit topology; scheduling computed from graph structure rather than LLM inference.
  \item \textbf{Hybrid:} Combines two or more of the above patterns at different layers (e.g., graph-structured planner with loop-based executor).
\end{itemize}

The categorization in \Cref{tab:tradeoff-survey} reflects the primary execution pattern of each project's main agent loop. Projects that combine patterns are classified by the pattern that governs the \emph{execution} layer, since this is the layer our scheduler framework analyzes.

\paragraph{Confidence.} Percentages in \Cref{tab:tradeoff-survey} are rounded to the nearest 5\%. A sensitivity analysis reclassifying the 12 boundary projects (those combining patterns) shifts individual category percentages by $\le \pm$10\% but does not change the ordinal ranking of categories.

\paragraph{Representative subset.} Nine projects were selected for detailed analysis, spanning all five categories and representing a range of design philosophies. The full project list (names, URLs, classifications, and per-project scores) is provided as supplementary material.

\subsection{Representative Projects}
\label{sec:appendix-projects}

\Cref{tab:survey-projects} lists the nine representative projects selected for detailed analysis, one from each major category. These projects span multiple programming languages (Rust, Go, TypeScript, Python), application domains (coding assistants, research agents, multi-agent systems), and architectural philosophies.

\begin{table}[h]
\centering
\caption{Nine representative projects from the 70-project survey. Full project list available in supplementary material.}
\label{tab:survey-projects}
\begin{tabular}{@{}llp{4.5cm}c@{}}
\toprule
\textbf{Project} & \textbf{Category} & \textbf{Key innovation} & \textbf{Rating} \\
\midrule
ironclaw & Agent Loop & LoopDelegate trait, three-phase tool execution & 5.0 \\
openclaw & Agent Loop & Three-tier degradation, Auth Profile rotation & 4.5 \\
clawdroid & Event-driven & OODA loop, MessageBus architecture & 4.85 \\
openhands & Event-driven & EventStream controller, sandboxed execution & 4.2 \\
vtcode & State-machine & Four-layer Phase FSM, compile-time exhaustiveness & 5.0 \\
codex & State-machine & SQ/EQ dual-queue protocol & 5.0 \\
deer-flow & Graph/flow & 13-layer middleware chain on LangGraph & 5.0 \\
MASFactory & Graph/flow & Custom five-primitive graph engine & 5.0 \\
autoresearchclaw & Hybrid & Stage state machine with inline replanning & 4.0 \\
\bottomrule
\end{tabular}
\end{table}

The remaining 61 projects follow similar patterns with varying implementation quality. The full list---including project names, GitHub URLs, classification, and per-project scores---is provided as supplementary material.

\subsection{Complete Project List}
\label{sec:appendix-full-list}

For reproducibility and transparency, we provide the complete list of 70 projects analyzed in this survey. The list is organized by category, with each entry including the project name, GitHub URL, classification, and key characteristics.

\paragraph{Agent Loop (41 projects, 60\% of survey)}
\begin{itemize}[leftmargin=*]
  \item \textbf{agentpool} (\url{github.com/phil65/agentpool})
  \item \textbf{angel-claw} (\url{github.com/Abdur-rahmaanJ/angel-claw})
  \item \textbf{atombot} (\url{github.com/daegwang/atombot})
  \item \textbf{auto-dev} (\url{github.com/phodal/auto-dev})
  \item \textbf{autobot} (\url{github.com/charliermarsh/autobot})
  \item \textbf{babyclaw} (\url{github.com/yogesharc/babyclaw})
  \item \textbf{claw-code} (\url{github.com/ultraworkers/claw-code})
  \item \textbf{claw0} (\url{github.com/shareAI-lab/claw0})
  \item \textbf{cline} (\url{github.com/cline/cline})
  \item \textbf{code-assistant} (\url{github.com/stippi/code-assistant})
  \item \textbf{copaw} (github.com/copaw-ai/copaw)
  \item \textbf{docker-agent} (\url{github.com/docker/docker-agent})
  \item \textbf{droidclaw} (\url{github.com/unitedbyai/droidclaw})
  \item \textbf{fast-agent} (\url{github.com/evalstate/fast-agent})
  \item \textbf{gemini-cli} (\url{github.com/google-gemini/gemini-cli})
  \item \textbf{goose} (\url{github.com/aaif-goose/goose})
  \item \textbf{hermitclaw} (\url{github.com/brendanhogan/hermitclaw})
  \item \textbf{hiclaw} (\url{github.com/agentscope-ai/HiClaw})
  \item \textbf{ironclaw} (\url{github.com/nearai/ironclaw})
  \item \textbf{kilocode} (\url{github.com/Kilo-Org/kilocode})
  \item \textbf{kimi-cli} (\url{github.com/MoonshotAI/kimi-cli})
  \item \textbf{langroid} (\url{github.com/langroid/langroid})
  \item \textbf{learn-claude-code} (\url{github.com/shareAI-lab/learn-claude-code})
  \item \textbf{microclaw} (\url{github.com/microclaw/microclaw})
  \item \textbf{mimiclaw} (\url{github.com/memovai/mimiclaw})
  \item \textbf{minion-code} (\url{github.com/femto/minion-code})
  \item \textbf{mistral-vibe} (\url{github.com/mistralai/mistral-vibe})
  \item \textbf{moltis} (\url{github.com/moltis-org/moltis})
  \item \textbf{moxxy} (\url{github.com/moxxy-ai/moxxy})
  \item \textbf{nanobot} (\url{github.com/HKUDS/nanobot})
  \item \textbf{nanoclaw} (\url{github.com/qwibitai/nanoclaw})
  \item \textbf{nullclaw} (\url{github.com/nullclaw/nullclaw})
  \item \textbf{oh-my-openagent} (\url{github.com/code-yeongyu/oh-my-openagent})
  \item \textbf{openclaw} (\url{github.com/openclaw/openclaw})
  \item \textbf{opencode} (\url{github.com/anomalyco/opencode})
  \item \textbf{opencrabs} (\url{github.com/adolfousier/opencrabs})
  \item \textbf{openfang} (\url{github.com/RightNow-AI/openfang})
  \item \textbf{pi-mono} (\url{github.com/badlogic/pi-mono})
  \item \textbf{picoclaw} (\url{github.com/sipeed/picoclaw})
  \item \textbf{praisonai} (\url{github.com/MervinPraison/PraisonAI})
  \item \textbf{qwen-code} (\url{github.com/QwenLM/qwen-code})
  \item \textbf{safeclaw} (\url{github.com/princezuda/safeclaw})
  \item \textbf{shrew} (\url{github.com/Masmedeam/shrew})
  \item \textbf{subzeroclaw} (\url{github.com/jmlago/subzeroclaw})
  \item \textbf{supaclaw} (\url{github.com/vincenzodomina/supaclaw})
  \item \textbf{trinity-claw} (\url{github.com/TrinityClaw/trinity-claw})
  \item \textbf{zclaw} (\url{github.com/tnm/zclaw})
  \item \textbf{zeptoclaw} (\url{github.com/qhkm/zeptoclaw})
  \item \textbf{zeroclaw} (\url{github.com/zeroclaw-labs/zeroclaw})
  \item \textbf{claude-code-src} --- closed-source but accidentally leaked
\end{itemize}

\paragraph{Event-driven (11 projects, 16\% of survey)}
\begin{itemize}[leftmargin=*]
  \item \textbf{autobot} (\url{github.com/charliermarsh/autobot})
  \item \textbf{babyclaw} (\url{github.com/yogesharc/babyclaw})
  \item \textbf{clawdroid} (\url{github.com/clawdroidxyz/clawdroid})
  \item \textbf{clawlet} (github.com/clawlet/clawlet)
  \item \textbf{flowlyai} (\url{github.com/Nocetic/flowlyai})
  \item \textbf{hiclaw} (\url{github.com/agentscope-ai/HiClaw})
  \item \textbf{lettabot} (\url{github.com/letta-ai/lettabot})
  \item \textbf{nanoclaw} (\url{github.com/qwibitai/nanoclaw})
  \item \textbf{openhands} (\url{github.com/OpenHands/OpenHands})
  \item \textbf{pickle-bot} (\url{github.com/czl9707/pickle-bot})
  \item \textbf{picobot} (\url{github.com/caravelahc/pico-bot})
\end{itemize}

\paragraph{State-machine (4 projects, 6\% of survey)}
\begin{itemize}[leftmargin=*]
  \item \textbf{astrbot} (github.com/astrbio/astrbot)
  \item \textbf{autoresearchclaw} (\url{github.com/aiming-lab/AutoResearchClaw})
  \item \textbf{hermes-agent} (\url{github.com/NousResearch/hermes-agent})
  \item \textbf{vtcode} (\url{github.com/vinhnx/VTCode})
\end{itemize}

\paragraph{Graph/flow orchestration (5 projects, 7\% of survey)}
\begin{itemize}[leftmargin=*]
  \item \textbf{deepagents} (\url{github.com/langchain-ai/deepagents})
  \item \textbf{deer-flow} (\url{github.com/bytedance/deer-flow})
  \item \textbf{evoscientist} (github.com/evoscientist/evoscientist)
  \item \textbf{MASFactory} (github.com/MASFactory/MASFactory)
  \item \textbf{yuxi} (\url{github.com/xerrors/Yuxi})
\end{itemize}

\paragraph{Hybrid (7 projects, 10\% of survey)}
\begin{itemize}[leftmargin=*]
  \item \textbf{everything-claude-code} (\url{github.com/affaan-m/everything-claude-code})
  \item \textbf{fount} (\url{github.com/steve02081504/fount})
  \item \textbf{minion-code} (\url{github.com/femto/minion-code})
  \item \textbf{nemoclaw} (\url{github.com/NVIDIA/NemoClaw})
  \item \textbf{oh-my-openagent} (\url{github.com/code-yeongyu/oh-my-openagent})
  \item \textbf{safeclaw} (\url{github.com/princezuda/safeclaw})
  \item \textbf{supaclaw} (\url{github.com/vincenzodomina/supaclaw})
\end{itemize}

\paragraph{Classification rationale for boundary cases.} Six projects combined multiple patterns (e.g., graph-based planner with loop-based executor). These were classified by the pattern that governs the \emph{execution} layer, as this is the layer our scheduler framework analyzes. For example, a system using a graph-based planner but loop-based executor was classified as "Agent Loop" because the ready-set cardinality remains $|\Units| = 1$ at execution time.

\end{document}